\documentclass{article}

\PassOptionsToPackage{numbers, compress}{natbib}
\usepackage[preprint]{neurips_2026}
\usepackage{mathrsfs}
\usepackage[utf8]{inputenc} 
\usepackage[T1]{fontenc}    
\usepackage{fix-cm} 
\usepackage{hyperref}       
\usepackage{url}            
\usepackage{booktabs}       
\usepackage{amsfonts}       
\usepackage{bm}             
\usepackage{amsmath}
\usepackage{mathrsfs}
\usepackage{amsthm}
\usepackage{amssymb}
\usepackage{mathtools}
\usepackage{enumitem}
\usepackage{multirow}
\usepackage{cancel}
\usepackage{wrapfig}
\usepackage{comment}
\usepackage{inconsolata} 
\usepackage{tikz-cd}
\usepackage[most]{tcolorbox}
\usetikzlibrary{matrix,arrows,arrows.meta,backgrounds,positioning,shapes,shapes.multipart,fit,calc,shadows}
\usepackage{cleveref}
\usepackage{thmtools}
\usepackage{nicefrac}       
\usepackage{microtype}    
\usepackage{listings}
\usepackage[dvipsnames]{xcolor}

\usepackage{listings}
\lstdefinelanguage{Lean}{
  keywords={theorem, lemma, def, where, by, have, show, exact, apply, intro, intros, rw, simp, ring, linarith, omega, norm_num, calc, fun, let, in, match, with, if, then, else, return, do, forall, exists, and, or, not, true, false, Type, Prop, Sort, try, simp_all, constructor, decide, native_decide, trivial, nlinarith, field_simp, ring_nf, exact?},
  keywordstyle=\color{blue!70!black}\bfseries,
  comment=[l]{--},
  morecomment=[s]{/-}{-/},
  commentstyle=\color{gray}\itshape,
  stringstyle=\color{orange},
  basicstyle=\ttfamily\small,
  breaklines=true,
  frame=single,
  rulecolor=\color{black!20},
  backgroundcolor=\color{gray!5},
  numbers=none,
}
\lstnewenvironment{lean}{\lstset{language=Lean}}{}
\newcommand{\leaninline}[1]{\lstinline[language=Lean]{#1}}

\setlist[itemize]{leftmargin=0.3in}
\setlist[enumerate]{leftmargin=0.3in}

\DeclareMathOperator{\ob}{\textnormal{ob}}

\definecolor{tacticStroke}{HTML}{E0E0E0}

\definecolor{lkwColor}{HTML}{1F4D80}
\definecolor{ltyColor}{HTML}{2D8077}
\definecolor{ltacColor}{HTML}{B85C00}
\newcommand{\lkw}[1]{{\color{lkwColor}\textbf{#1}}}
\newcommand{\lty}[1]{{\color{ltyColor}#1}}
\newcommand{\ltac}[1]{{\color{ltacColor}#1}}

\usepackage[most]{tcolorbox}
\newtcolorbox{samplebox}[1]{
	enhanced,
	colframe=black!60,
	boxrule=0.6pt,
	arc=6pt,
	left=4pt, right=4pt, top=4pt, bottom=4pt,
	boxsep=3pt,
	before upper={\noindent\textbf{\small #1}\par\smallskip},
}

\newtheorem{theorem}{Theorem}

\newtheorem{lemma}[theorem]{Lemma}
\newtheorem{claim}[theorem]{Claim}
\newtheorem{proposition}[theorem]{Proposition}
\newtheorem{corollary}[theorem]{Corollary}
\theoremstyle{definition}
\newtheorem{definition}[theorem]{Definition}
\newtheorem{remark}[theorem]{Remark}
\newtheorem{example}[theorem]{Example}
\newtheorem{assumption}[theorem]{Assumption}

\newcommand{\draft}[1]{\textcolor{black}{#1}}

\newcommand{\Stmt}{\mathsf{Stmt}}
\newcommand{\Set}{\mathbf{Set}}

\newcommand{\RT}{\mathscr{R}_\mathcal{T}}
\newcommand{\DpT}{\mathbb{D}_+(\mathcal{T}^*)}
\newcommand{\PTR}{P_{[T]_\mathscr{R}}}

\definecolor{mini}{RGB}{120,80,200}
\definecolor{proof}{RGB}{60,120,220}
\definecolor{ineq}{RGB}{60,150,90}

\newtcolorbox{datasetbox}[2][]{
  enhanced,
  width=#2,
  colback=white,
  colframe=#1,
  boxrule=0.8pt,
  arc=6pt,
  left=6pt,right=6pt,top=6pt,bottom=6pt
}

\title{What are the Right Symmetries\\for Formal Theorem Proving?}

\author{
Krzysztof Olejniczak\textnormal{\textsuperscript{1}\quad}
Radoslav Dimitrov\quad
Xingyue Huang\textnormal{\textsuperscript{1}\quad}\\
\textbf{Bernardo Cuenca Grau}\textnormal{\textsuperscript{1}\quad}
\textbf{Jinwoo Kim}\thanks{Equal advising.}\textnormal{\:\:\:\textsuperscript{2}\quad}
\textbf{{\.I}smail {\.I}lkan Ceylan}\textnormal{\textsuperscript{* 3,4,1}}\\
\textsuperscript{1}University of Oxford\quad
\textsuperscript{2}KAIST\quad
\textsuperscript{3}TU Wien\quad
\textsuperscript{4}AITHYRA\quad
}

\begin{document}

\maketitle

\begin{abstract}
Formal theorem provers based on large language models (LLMs) are highly sensitive to superficial variations in problem representation: semantically equivalent statements can exhibit drastically different proof success rates, revealing a failure to respect structural symmetries inherent in formal mathematics. This raises a central question: \emph{what are the right symmetries for formal theorem proving?}
We introduce \emph{rewriting categories}, a category-theoretic framework capturing the compositional, generally non-invertible transformations induced by proof tactics, and use it to formalize two symmetry notions: \emph{proof equivariance}, governing how proof distributions transform under rewrites, and \emph{success invariance} (i.e., invariance of success probability), requiring equivalent statements to be solved with the same probability. 
\draft{We observe that state-based next-tactic provers naturally satisfy proof equivariance by operating on proof states.}
In contrast, state-of-the-art LLM-based provers satisfy neither property, exhibiting large performance variation across equivalent formulations.
To mitigate this, we propose test-time methods that aggregate over equivalent rewritings of the input, showing theoretically that they recover success invariance in the sampling limit, and empirically, that they improve robustness and performance under fixed inference budgets. Our results highlight symmetry as a key missing inductive bias in LLM-based theorem proving and suggest test-time computation as a practical route to approximate it.
\end{abstract}
\begin{figure}[h]
    \centering
    \vspace{-0.2em}
%
\definecolor{gcdTitle}{HTML}{2D5BBF}      
\definecolor{gcdSucc}{HTML}{2A9C46}       
\definecolor{gcdFail}{HTML}{C13A3A}       
\definecolor{gcdMid}{HTML}{C19A00}        

\newcommand{\gcdCell}[3]{%
    \begin{minipage}[t][15mm][t]{\linewidth}%
        {\ttfamily\scriptsize #1}%
        \vfill
        \hfil{\sffamily\bfseries\scriptsize\color{#2}#3}\hfil%
    \end{minipage}%
}

\resizebox{\linewidth}{!}{%
\begin{tikzpicture}[
    font=\small,
    hdr/.style={
        draw=gcdTitle!75, fill=gcdTitle!16, text=gcdTitle,
        rounded corners=3pt, line width=0.5pt,
        font=\sffamily\bfseries\footnotesize, align=center,
        minimum height=2.4em, inner sep=4pt,
        drop shadow={shadow xshift=0pt, shadow yshift=-0.7pt, opacity=0.18, fill=black!60},
    },
    cell/.style={
        draw, rounded corners=3pt, line width=0.6pt,
        fill=white, align=left,
        minimum height=21mm,
        inner sep=5pt,
        drop shadow={shadow xshift=0pt, shadow yshift=-1pt, opacity=0.20, fill=black!60},
    },
    succCell/.style={cell, draw=gcdSucc!75},
    midCell/.style ={cell, draw=gcdMid!85},
    failCell/.style={cell, draw=gcdFail!75},
    narrow/.style={text width=0.27\linewidth},
    wide/.style={text width=0.43\linewidth},
]
\matrix[
    matrix of nodes,
    ampersand replacement=\&,
    column sep=3pt, row sep=4pt,
    nodes={anchor=north},
]{
    |[hdr, narrow]| {1) Commuting arguments}
    \&
    |[hdr, wide  ]| {2) Algebraic transformations}
    \&
    |[hdr, narrow]| {3) Alternative formulations}
    \\
    |[succCell, narrow]|  {\gcdCell
        {\lkw{theorem} number\_theory\_1 :\\
         \hspace*{1em}Nat.gcd 200000 20! = 40000}
        {gcdSucc}{Success rate: 63/64 (98.4\%)}}
    \&
    |[succCell,  wide  ]| {\gcdCell
        {\lkw{theorem} algebra\_1 (n : \lty{$\mathbb{N}$}) (h$_0$ : n \% 2 = 0)\\
         \hspace*{1em}(h$_1$ : (n-2)*(n-2) + n*n + (n+2)*(n+2) = 12296)\\
         \hspace*{1em}: (n * (n+2) * (n-2)) / 8 = 32736}
        {gcdSucc}{Success rate: 32/64 (50.0\%)}}
    \&
    |[succCell, narrow]| {\gcdCell
        {\lkw{theorem} induction\_1 (n : \lty{$\mathbb{N}$})\\
         \hspace*{1em}: (20 + 4 * 4\textasciicircum n) \% 12 = 0}
        {gcdSucc}{Success rate: 62/64 (96.9\%)}}
    \\
    |[failCell, narrow]|  {\gcdCell
        {\lkw{theorem} number\_theory\_2 :\\
         \hspace*{1em}Nat.gcd 20! 200000 = 40000}
        {gcdFail}{Success rate: 44/64 (68.9\%)}}
    \&
    |[failCell, wide  ]| {\gcdCell
        {\lkw{theorem} algebra\_2 (n : \lty{$\mathbb{N}$}) (h$_0$ : \lty{Even} n)\\
         \hspace*{1em}(h$_1$ : (n-2) \textasciicircum\: 2 + n \textasciicircum\:2 + (n+2) \textasciicircum\:2 = 12296)\\
         \hspace*{1em}: (n-2) * n * (n+2) / 8 = 32736}
        {gcdFail}{Success rate: 3/64 (4.7\%)}}
    \&
    |[failCell, narrow]| {\gcdCell
        {\lkw{theorem} induction\_2 (n : \lty{$\mathbb{N}$})\\
         \hspace*{1em}: 12 \textbar\ 4\textasciicircum (n + 1) + 20}
        {gcdFail}{Success rate: 13/64 (20.3\%)}}
    \\
};
\end{tikzpicture}%
}
    \vspace{-1em}
     \caption{
    Sensitivity of DeepSeek-Prover-V2 \cite{ren2025deepseek} to semantics-preserving rewrites. Each pair of statements is mathematically equivalent, yet exhibits drastically different proof success rates under minor transformations. This highlights a systematic violation of symmetry: the prover’s success probability depends strongly on the representation instead of the underlying mathematical content.
     }
    \label{fig:gcd_example}
\end{figure}

\section{Introduction}
\vspace{-0.5em}
Large language model (LLM) theorem provers, operating on textual representations of formal statements, are highly sensitive to how problems are expressed \citep{zhao2025ineq, tian2025evolprover}.
As shown in Figure~\ref{fig:gcd_example}, even simple, semantics-preserving rewrites, such as reordering arguments, applying elementary algebraic transformations, or alternative notation, can lead to drastic changes in proof success rates.
In some cases, a theorem that is solved nearly always becomes almost unsolvable under another formulation, despite the two statements being mathematically equivalent.
This behavior, widely observed in mathematical reasoning and language understanding \citep{mirzadeh2024gsm, irpan2025consistency, berglund2023reversal, wu2025rewordbench, wei2025equibench}, reveals a fundamental limitation of current systems: failure to generalize across different representations of the same problem.

\begin{wrapfigure}{r}{0.42\textwidth}
    \centering
    \vspace{-0.5\baselineskip}
%

\definecolor{nitThmHdr}{HTML}{ECB8E8}
\definecolor{nitThmHdrText}{HTML}{6E2A85}
\definecolor{nitThmDraw}{HTML}{B776BC}

\resizebox{0.8\linewidth}{!}{%
\begin{tikzpicture}[
    font=\sffamily\normalsize,
    >={Stealth[length=2.6mm,width=2.0mm]},
    thmbox/.style={
        rectangle split, rectangle split parts=2,
        rectangle split part fill={nitThmHdr, white},
        rectangle split draw splits=true,
        rounded corners=3pt, draw=nitThmDraw, line width=0.4pt,
        fill=white,
        text width=6.2cm, inner xsep=6pt, inner ysep=4pt,
        align=left,
        font=\ttfamily\footnotesize,
        text=black,
    },
    tacticpill/.style={
        draw=orange!60!black!50, rounded corners=4pt, fill=yellow!35, line width=0.4pt,
        inner sep=5pt, font=\ttfamily\footnotesize,
        drop shadow={shadow xshift=0pt, shadow yshift=-0.7pt, opacity=0.18, fill=black!60},
    },
]
    \node[thmbox] (orig) at (0,0) {%
        {\sffamily\bfseries\footnotesize\color{nitThmHdrText}Original statement}
        \nodepart{two}%
        \lkw{theorem} T (x : \lty{$\mathbb{R}$}) (h : x > 1) : \ldots \lkw{:= by}};

    \node[thmbox, below=1.5cm of orig] (aug) {%
        {\sffamily\bfseries\footnotesize\color{nitThmHdrText}Updated statement}
        \nodepart{two}%
        \lkw{theorem} T (x : \lty{$\mathbb{R}$}) (h : x > 0) : \ldots \lkw{:= by}};

    \path (orig.south) -- coordinate[midway] (midarr) (aug.north);
    \draw[->, line width=0.7pt] (orig.south) -- (aug.north);

    \node[tacticpill, anchor=center] at ($(midarr)$) {%
        \ltac{replace} h : x > 0 \lkw{:=} \lkw{by} linarith};
\end{tikzpicture}%
}
    \caption{An example of tactic irreversibility in Lean. Applying a tactic transforms $x>1$ into $x>0$, losing information.}
    \label{fig:statement_augmentation}
    \vspace{-0.5\baselineskip}
\end{wrapfigure}
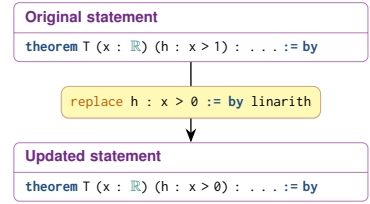
LLM-based provers \citep{polu2020generative, yang2023leandojo, xin2025deepseekproverv} process statements as text, making their behavior depend on surface form instead of the underlying mathematical structure.
This discrepancy motivates a central question: \emph{what are the right symmetries for formal theorem proving?}
In classical machine learning, symmetries are typically modeled via group actions \citep{bronstein2021geometric}, which capture invertible transformations such as rotations or permutations. However, this framework is insufficient for formal reasoning.
In proof assistants such as Lean \citep{moura2021lean}, transformations between statements arise via sequences of tactics, which are generally non-invertible and may discard information (see \Cref{fig:statement_augmentation}).
Consequently, the resulting transformation space cannot be characterized by group symmetries alone.

In this work, we introduce \emph{rewriting categories}, a category-theoretic framework for modeling transformations between theorem statements. 
This perspective allows us to formalize two symmetry notions for formal theorem provers: \textbf{proof equivariance}, which requires proof distributions to transform consistently along rewrites, and \textbf{success invariance} (i.e., invariance of success probability), which requires equivalent statements to be solved with the same probability.
\draft{Proof equivariance is naturally satisfied by next-tactic predictors~\citep{yang2023leandojo, li2024hunyuanprover, xin2025bfsprover} when used generate proofs sequentially from intermediate proof states.}
On the other hand, modern state-of-the-art LLM-based provers~\cite{ren2025deepseek, lin2025goedel, lin2025goedel_v2, kimina_prover_2025} satisfy neither property, exhibiting large variability across equivalent formulations.

We show that this lack of symmetry can be mitigated by aggregating over equivalent rewritings of the input statement.
We prove that such procedures recover success invariance in the sampling limit, and empirically demonstrate improved robustness and performance under fixed inference budgets. These results suggest that symmetry, defined in a categorical sense, is a key missing inductive bias in current LLM-based theorem provers, and that it can be effectively approximated through test-time computation without modifying the underlying model.
Our contributions can be summarized as\footnote{Code available at: \href{https://github.com/kolejnyy/rw-ensembles}{https://github.com/kolejnyy/rw-ensembles}}:

\begin{enumerate}[noitemsep,nolistsep,label=(\arabic*),leftmargin=*]
    \item
    \textbf{Category-theoretic framework.} We introduce rewriting categories, a category-theoretic framework for modeling transformations between formal theorem statements, and formalize two symmetry notions: proof equivariance and success invariance.
    \item \textbf{Empirical analysis of invariance.} We construct miniF2F-rw, a benchmark consisting of semantically equivalent reformulations of formal problems, and show that state-of-the-art LLM-based provers exhibit large violations of success invariance.
    \item \textbf{Test-time invariance.} We propose a simple, model-agnostic test-time procedure that aggregates over equivalent rewritings of the input, and theoretically prove that this approach recovers success invariance in the sampling limit. 
    \item \textbf{Competitive performance.} We demonstrate empirically that the proposed method improves both robustness to input reformulations and overall proof success under fixed inference budgets.
\end{enumerate}
\vspace{-0.5em}
\section{Formal theorem proving in Lean}
\label{sec:formal_theorem_proving}

\noindent\textbf{Theorem proving in Lean.}
Formal theorem proving in Lean is structured around transformations of proof states \citep{moura2021lean}. A theorem statement is compiled into an initial state representing assumptions and goals, and proofs are constructed by applying sequences of tactics that transform this state until no goals remain. This operational view exposes a key structural property: transformations between statements arise through sequences of tactics, which are compositional but generally non-invertible.

Formally, let $\Stmt$ denote the set of all valid theorem statements, $\mathcal{T}$ the set of all tactics, and $\Omega$ the set of all Lean states, with $\Omega$  containing an \emph{empty state} $\varnothing$ (no remaining goals) and an \emph{error state} $\bot$ (failure during compilation).
Given a theorem $T$, we will write $\mathcal{L}(T)$ for the corresponding state.

Each tactic $t \in \mathcal{T}$ can be seen as a function $t:\Omega \to \Omega$.
In this view, a sequence of tactics $\mathbf{t}=(t_1, \dots, t_k)$ is a valid proof of a statement $T$, if their composition maps $\mathcal{L}(T)$ to the empty state:
\[
{\bf t}(\mathcal{L}(T)) = (t_k \circ t_{k-1} \circ \cdots \circ t_1)(\mathcal{L}(T)) = \varnothing.
\]

\textbf{Formal LLM-based provers.}
We focus on single-pass whole-proof generation models \citep{ren2025deepseek,lin2025goedel,lin2025goedel_v2, kimina_prover_2025}. In this setting, an LLM defines a distribution $p_\theta(\cdot \mid T)$ over sequences of tactics (i.e., candidate proofs), conditioned on a Lean statement $T$. The~success rate $s_\theta(T)$ is then the probability that a sampled sequence of tactics $\mathbf{t}$ transforms the initial state $\mathcal{L}(T)$ into $\varnothing$:
\[
s_\theta(T) \;\coloneqq\; \sum_{{\bf t}} p_\theta({\bf t} \mid T)\, \mathbf{1}\!\left({\bf t}(\mathcal{L}(T)) = \varnothing\right)
\]
Importantly, both the input theorems and output proofs are represented as text.

\begin{figure}[t]
    \vspace{-2em}
    \centering
%

\definecolor{nitThmHdr}{HTML}{ECB8E8}
\definecolor{nitThmHdrText}{HTML}{6E2A85}
\definecolor{nitThmDraw}{HTML}{B776BC}

\definecolor{nitPrfHdr}{HTML}{800010}

\definecolor{nitLemHdr}{HTML}{B6DFBF}
\definecolor{nitLemHdrText}{HTML}{1F6938}
\definecolor{nitLemDraw}{HTML}{6BB07F}

\definecolor{nitStateHdr}{HTML}{D8E2F0}
\definecolor{nitStateHdrText}{HTML}{1F4D80}
\definecolor{nitStateDraw}{HTML}{9CB4D2}
\definecolor{nitStateHl}{HTML}{1F4D80}
\definecolor{nitStateRule}{HTML}{C5CDD7}

\resizebox{\linewidth}{!}{%
\begin{tikzpicture}[
    font=\sffamily\small,
    >={Stealth[length=2.6mm,width=2.0mm]},
    thmbox/.style={
        rectangle split, rectangle split parts=2,
        rectangle split part fill={nitThmHdr, white},
        rectangle split draw splits=true,
        rounded corners=3pt, draw=nitThmDraw, line width=0.4pt,
        fill=white,
        text width=5.55cm, inner xsep=6pt, inner ysep=4pt,
        align=left,
        font=\ttfamily\scriptsize,
        text=black,
    },
    lembox/.style={
        rectangle split, rectangle split parts=2,
        rectangle split part fill={nitLemHdr, white},
        rectangle split draw splits=true,
        rounded corners=3pt, draw=nitLemDraw, line width=0.4pt,
        fill=white,
        text width=4.05cm, inner xsep=6pt, inner ysep=4pt,
        align=left,
        font=\ttfamily\scriptsize,
        text=black,
    },
    statebox/.style={
        rectangle split, rectangle split parts=2,
        rectangle split part fill={nitStateHdr, white},
        rectangle split draw splits=true,
        rounded corners=3pt, draw=nitStateDraw, line width=0.4pt,
        fill=white,
        text width=3.05cm, inner xsep=6pt, inner ysep=4pt,
        align=left,
        font=\ttfamily\scriptsize,
        text=black,
    },
    proofbox/.style={
        rectangle split, rectangle split parts=2,
        rectangle split part fill={nitPrfHdr, white},
        rectangle split draw splits=true,
        rounded corners=3pt, draw=nitPrfHdr, line width=0.4pt,
        fill=white,
        text width=5.55cm, inner xsep=6pt, inner ysep=4pt,
        align=left,
        font=\ttfamily\scriptsize,
        text=black,
    },
    arrlbl/.style={font=\sffamily\scriptsize\itshape, fill=white, inner sep=2pt},
    tacticpill/.style={
        draw=orange!60!black!50, rounded corners=4pt, fill=yellow!35, line width=0.4pt,
        inner sep=4pt, font=\ttfamily\scriptsize,
        align=center,
        drop shadow={shadow xshift=0pt, shadow yshift=-0.7pt, opacity=0.18, fill=black!60},
    },
]
    \node[thmbox] (t1) at (0,0) {%
        {\sffamily\bfseries\scriptsize\color{nitThmHdrText}Lean theorem statement \(T_1\in\Stmt\)}
        \nodepart{two}%
        \lkw{theorem} thm1 (x y : \lty{$\mathbb{N}$}) (h\textsubscript{0} : Nat.lcm x y = 12)\\
        \mbox{}\hspace*{1.2em}(h\textsubscript{1} : Nat.gcd x y = 2) : x * y = 24 \lkw{:= by}};

    \node[lembox] (lem) at (5.9,0) {%
        {\sffamily\bfseries\scriptsize\color{nitLemHdrText}Mathlib lemma}
        \nodepart{two}%
        \lkw{theorem} Nat.gcd\_mul\_lcm (x y : \lty{$\mathbb{N}$})\\
        \mbox{}\hspace*{1.2em}: x.gcd y * x.lcm y = x * y};

    \node[proofbox, anchor=north west] (prf) at ($(lem.north east)+(1.1cm,0)$) {
        {\sffamily\bfseries\scriptsize\color{nitPrfHdr}Resulting proof}
        \nodepart{two}%
        \lkw{theorem} thm1 (x y : \lty{$\mathbb{N}$}) (h\textsubscript{0} : Nat.lcm x y = 12)\\
        \mbox{}\hspace*{1.2em}(h\textsubscript{1} : Nat.gcd x y = 2) : x * y = 24 \lkw{:= by}\\
        \mbox{}\hspace*{1.2em}rw [\(\leftarrow\) Nat.gcd\_mul\_lcm x y]\\
        \mbox{}\hspace*{1.2em}norm\_num [h\textsubscript{0}, h\textsubscript{1}]
    };

    \node[statebox, below=1.2cm of t1, text width=3.25cm] (s1) {%
        {\sffamily\bfseries\scriptsize\color{nitStateHdrText}Compiled state \(\mathcal{L}(T_1)\in\Omega\)}
        \nodepart{two}%
        x y : \(\mathbb{N}\)\\
        {\color{nitStateHl}h\textsubscript{0} : x.lcm y = 12}\\
        {\color{nitStateHl}h\textsubscript{1} : x.gcd y = 2}\\[-3pt]
        {\color{nitStateRule}\rule{\linewidth}{0.3pt}}\\[1pt]
        \(\vdash\) x * y = 24};

    \node[statebox, anchor=north west] at ($(s1.north east)+(4.4cm,0)$) (s2) {%
        {\sffamily\bfseries\scriptsize\color{nitStateHdrText}Updated state}
        \nodepart{two}%
        x y : \(\mathbb{N}\)\\
        {\color{nitStateHl}h\textsubscript{0} : x.lcm y = 12}\\
        {\color{nitStateHl}h\textsubscript{1} : x.gcd y = 2}\\[-3pt]
        {\color{nitStateRule}\rule{\linewidth}{0.3pt}}\\[1pt]
        \(\vdash\) x.gcd y * x.lcm y = 24};

    \node[statebox, anchor=west, text width=2.05cm] at ($(s2.east)+(3.05cm,0)$) (s3) {%
        {\sffamily\bfseries\scriptsize\color{nitStateHdrText}Updated state}
        \nodepart{two}%
        no goals};

    \draw[->, dashed, line width=0.5pt]
        (t1.south) -- node[arrlbl, right=1pt] {compile} (s1.north);

    \draw[->, line width=0.7pt] (s1.east) -- (s2.west);
    \coordinate (m12) at ($(s1.east)!0.5!(s2.west)$);
    \node[tacticpill, above=5pt of m12] {%
        \ltac{rw} [\(\leftarrow\) Nat.gcd\_mul\_lcm x y]};

    \draw[->, line width=0.7pt] (s2.east) -- (s3.west);
    \coordinate (m23) at ($(s2.east)!0.5!(s3.west)$);
    \node[tacticpill, above=5pt of m23] {%
        \ltac{norm\_num} [h$_0$, h$_1$]};
\end{tikzpicture}%
}
    \vspace{-1em}
    \caption{An example of a successful proof in Lean. The original statement $T$ is converted into an~initial state~$\mathcal{L}(T)$, and~tactics are then sequentially executed, until the empty state $\varnothing$ is reached.}
    \label{fig:gcd_mul_lcm_proof}
\end{figure}
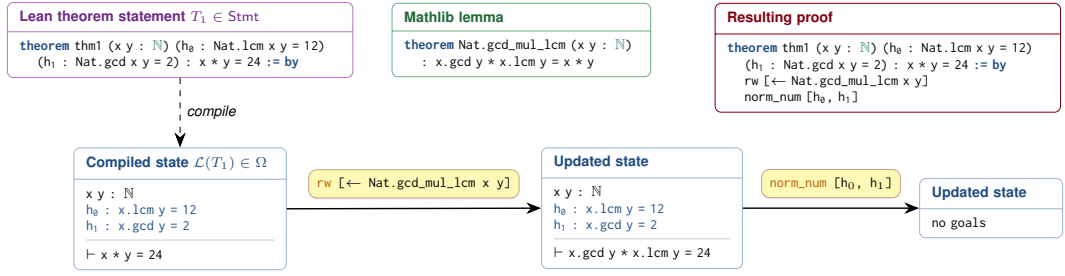

\textbf{The structure of Lean statement space.}
In Lean, transformations between statements arise through sequences of tactics acting on proof states, linking statements $T$ and $T'$ whenever $\mathcal{L}(T') = \mathbf{t}(\mathcal{L}(T))$. While it is tempting to model this structure via group actions, such an abstraction is inappropriate. These transformations are compositional but generally non-invertible: multiple statements may map to the same state, and tactic application can discard information, preventing reversal. As a result, the transformation structure cannot be captured by group symmetries.

\section{Formalization of mathematical symmetries via rewriting categories}
\label{sec:formalization_lean_space}

We now formalize the transformations described in \Cref{sec:formal_theorem_proving} using category theory. We model theorem statements as objects and transformations between them as composable, generally non-invertible morphisms. This induces a categorical structure that captures the behavior of tactics, which we refer to as a rewriting category. This framework provides the foundation for defining symmetry in formal theorem proving and explains inductive biases present in prior works as special cases.

\subsection{Category theory preliminaries}
\label{sec:category_theory_preliminaries}
We recall the basic categorical notions needed for our construction; see Appendix \ref{app:extended_category} for more details.
\begin{definition}[\bf Category]
\label{def:category}
A \emph{category} $\mathscr{C}$ consists of
\vspace{-0.3em}
\begin{itemize}[noitemsep,nolistsep]
\item a collection of \emph{objects} $\ob(\mathscr{C})$,
\item for each $A,B\in\ob(\mathscr{C})$, a collection $\mathscr{C}(A,B)$ of \emph{maps} or \emph{arrows} from $A$ to $B$,
\item for each pair of arrows $f\in\mathscr{C}(A,B)$ and $g\in \mathscr{C}(B,C)$, their \emph{composition} $g\circ f\in\mathscr{C}(A,C)$,
\item for each object $A\in \ob(\mathscr{C})$, an element $1_A$ of $\mathscr{C}(A,A)$, called the \emph{identity} on $A$,
\end{itemize}
\vspace{-0.3em}
satisfying the following axioms:
\vspace{-0.3em}
\begin{itemize}[noitemsep,nolistsep]
\item \emph{associativity}: $(h\circ g)\circ f= h\circ(g\circ f)$ for any $f\in \mathscr{C}(A,B)$, $g\in \mathscr{C}(B,C)$ and $h\in \mathscr{C}(C,D)$,
\item \emph{identity laws}: for each $f\in \mathscr{C}(A,B)$, we have $f\circ1_A = f = 1_B\circ f$.
\end{itemize}
\end{definition}
We often write $A\in\mathscr{C}$ to mean $A\in\ob(\mathscr{C})$, and $f:A\to B$ or $A \xrightarrow{f} B$ to mean $f\in \mathscr{C}(A,B)$.
We will also write $A \leftrightarrows B$ for objects $A,B\in \mathscr{C}$ such that there exist arrows $A\rightarrow B$ and $B\rightarrow A$ in $\mathscr{C}$. 
When for all objects $A,B\in\mathscr{C}$, there is an arrow $A\to B$ if and only if there is an arrow $B\to A$, we call $\mathscr{C}$ \emph{reciprocal}.
If categories $\mathscr{C}$ and $\mathscr{D}$ satisfy $\ob(\mathscr{C}) \subseteq \ob(\mathscr{D})$, and for every $A,B\in \mathscr{C}$, we have $\mathscr{C}(A,B) \subseteq \mathscr{D}(A,B)$, we call $\mathscr{C}$ a \emph{subcategory} of $\mathscr{D}$, denoted as $\mathscr{C} \subseteq \mathscr{D}$.
\begin{example}\label{example:category_of_sets}
There is a category $\mathbf{Set}$ whose objects are sets, whose maps are functions between sets, with function composition as composition and identity functions as identities. 
\end{example}
The following example instantiates this structure for Lean theorem proving.
\begin{example}\label{example:category_of_lean}
There is a \emph{full tactic rewriting category} $\RT$, whose objects are Lean statements $\Stmt$ and arrows represent sequences of tactics. More precisely, if statements $T,T'\in \Stmt$ satisfy $\mathcal{L}(T') = \mathbf{t}(\mathcal{L}(T))$ for some sequence of tactics $\bf t$, there is an arrow $\mathbf{t} : T\rightarrow T'$.
Identities correspond to the empty sequence of tactics,
representing no action. Composition becomes concatenation:
\[
\mathbf{t}' \circ \mathbf{t} = [\mathbf{t} : \mathbf{t}'] \qquad \text{whenever}\qquad T \xrightarrow{\mathbf{t}} T' \xrightarrow{\mathbf{t}'} T'' 
\]
\end{example}

To let categories act on concrete data, we use structure-preserving maps, called \emph{functors}:
\begin{definition}[\bf Functor]
\label{def:functor}
A functor $F:\mathscr{C}\to\mathscr{D}$ consists of an object map $F:\ob(\mathscr{C})\to\ob(\mathscr{D})$, and, for each pair $A,B\in\ob(\mathscr{C})$, an assignment on arrows, $F:\mathscr{C}(A,B)\to\mathscr{D}(F(A),F(B))$, satisfying:
\begin{itemize}[noitemsep,nolistsep]
\vspace{-0.5em}
\item \emph{preservation of composition}: $F(g\circ f)=F(g)\circ F(f)$ whenever $A\xrightarrow{f}B\xrightarrow{g}C$ in $\mathscr{C}$,
\item \emph{preservation of identities}: $F(1_A)=1_{F(A)}$ for every $A\in\mathscr{C}$.
\end{itemize}
\end{definition}
\begin{example}
\label{example:implementation-functor}
    There is an \emph{implementation functor} $\mathcal{I} : \RT \to \Set$ mapping each theorem $T\in \Stmt$ to the singleton set $\{T\}$, while arrows are mapped to the unique functions between singletons.
\end{example}

\subsection{Rewriting categories}
\label{sec:rewriting_categories}

With the language of \Cref{sec:category_theory_preliminaries} in hand, we introduce the specific categorical objects we will use to reason about semantics-preserving rewrites of theorem statements.
Our goal is to define categories $\mathscr{R}$ with Lean statements as objects, and whose arrows encode their transformations induced by tactics.

\textbf{The maximal semantic rewriting category.}
Every $T\in\mathsf{Stmt}$ carries intended mathematical content, and we posit an abstract \emph{maximal semantic rewriting category} $\mathscr{R}^*$ whose objects are statements, and arrows are all mathematically-sound transformations (including, but not limited to Lean tactics).
Concretely, there is an arrow $a : T \rightarrow T'$ in $\mathscr{R}^*$ for any logically valid argument $a$ reducing $T$ to $T'$.

\textbf{Rewriting categories for Lean.}
In the general case, however, we do not have explicit access to $\mathscr{R}^*$ and hence consider subcategories of the tractable subcategory $\mathscr{R}_{\mathcal{T}}\subseteq\mathscr{R}^*$ from \Cref{example:category_of_lean}.
\begin{definition}\label{def:rewriting_category}
    A \emph{(Lean) rewriting category} is a subcategory $\mathscr{R}\subseteq \RT$ of $\RT$ with $\ob(\mathscr{R})=\Stmt$.
\end{definition}
Of particular interest are rewriting categories $\mathscr{R}$ that restrict the allowed set of tactics.
For a set $\mathcal{T}'\subseteq \mathcal{T}$, we say that $\mathscr{R}$ is \emph{generated by $\mathcal{T}'$} if $\mathscr{R}$ is the inclusion-wise minimal rewriting category, such that for all $T,T'\!\in\!\Stmt \text{ and }t\in\mathcal{T}'$ with $\mathcal{L}(T') = t(\mathcal{L}(T))$, there is an arrow $t : T\to T'$ in~$\mathscr{R}$.

Each such category encodes a specific, increasingly conservative notion of available transformations.
This hierarchy of subcategory approximations of $\mathscr{R}^*$ is used to scaffold our empirical study in \Cref{sec:methodology}, where different rule sets yield augmentations of varying semantic strength, and the sensitivity of a prover to the choice measures how far its behavior is from the invariance we desire.

With the rewriting category $\mathscr{R}_{\mathcal{T}}$ in place, and its restrictions $\mathscr{R} \subseteq \mathscr{R}_{\mathcal{T}}$, we are ready to define the structural biases expected from formal provers.
We begin with the notion of \emph{$\mathscr{R}$-equivariance of the proof distribution}, ensuring that the structures of solutions produced by the prover align with $\mathscr{R}$.

\subsection{Proof equivariance}
\label{sec:proof_equivariance}

If for a pair of statements $T, T' \in \Stmt$, a tactic $t \in \mathcal{T}$ transforms $\mathcal{L}(T)$ into $\mathcal{L}(T')$, then any proof $\mathbf{t}$ of $T'$ can be lifted to a proof of $T$ by concatenating $t$ at the beginning, i.e., $[t : \mathbf{t}]$.
This observation motivates an equivariance property for formal provers $p_\theta$. Formally, this corresponds to transporting distributions along arrows in the rewriting category.
Intuitively, the distribution $p_\theta([t: \cdot] \mid T)$ of completions of proofs of $T$ that start with $t$ should match the distribution of proofs for $T'$.
If a prover $p_\theta$ meets this condition, improving its performance on $T'$ would automatically translate to increased performance on $T$, for proofs that start with $t$.
To ensure that this distribution transport is well-defined and unambiguous, we will assume the following:

\begin{assumption}
\label{ass:non-zero-proof-chance}
    Denote by $\mathcal{T}^*$ the space of finite sequences of Lean tactics, and let $\DpT$ be the set of distributions over $\mathcal{T}^*$, assigning non-zero probability to each element $\mathbf{t}\in\mathcal{T}^*$.
    Let~$T\in\Stmt$. 
    Given $T$, $p_\theta$ has a non-zero probability of producing any proof $\mathbf{t}\in\mathcal{T}^*$, i.e. $p_\theta(\cdot\mid T) \in \DpT$.
\end{assumption}

\Cref{ass:non-zero-proof-chance} holds, in principle, for various classes of models, including LLMs, relying on logit-based next-token sampling.
The concept of proof transport can now be defined in terms of functors:

\begin{definition}[\bf Proof functor]\label{def:proof_functor}
The \emph{(stochastic) proof functor} on $\mathscr{R}$ is the functor $\mathcal{P}:\mathscr{R}\to\mathbf{Set}$ that maps each object $T\in \mathscr{R}$ to $\mathcal{P}(T) = \mathbb{D}_+(\mathcal{T}^*)$,
and whose action on arrows $\mathbf{t} : T \rightarrow T'$ is defined as:
\[
[\mathcal{P}(\mathbf{t})](\nu)(\mathbf{t}') \propto \nu([\mathbf{t}\!:\! \mathbf{t}'])\qquad \forall \nu \in \mathbb{D}_+(\mathcal{T}^*), \mathbf{t}'\in \mathcal{T}^* 
\]
\end{definition}
In words, for any arrow $\mathbf{t}:T\to T'$ in $\mathscr{R}$, $\mathcal{P}(\mathbf{t}) : \DpT \to \DpT$ is a function mapping each distribution $\nu \in \DpT$ to its restriction over proofs that start with $\bf t$.
The proof of functoriality of $\mathcal{P}$ is presented in Appendix \ref{sec:proof-functoriality}.
Using proof functors, we can formalize the property outlined above:
\begin{definition}[\bf $\mathscr{R}$-equivariance of proof distribution]\label{def:strong_consistency}
Let $\mathcal{P}:\mathscr{R}\to\mathbf{Set}$ be the proof functor.
A~prover $p_\theta$ is \emph{$\mathscr{R}$-equivariant in proof distribution} if
\[
p_\theta\left(\cdot \mid T'\right) = \mathcal{P}(\mathbf{t})(p_\theta(\cdot \mid T))
\qquad\text{for every arrow }\mathbf{t}:T\to T'\text{ in }\mathscr{R}.
\]
\end{definition}
Under \Cref{ass:non-zero-proof-chance}, this is a well-typed definition.
$\mathscr{R}$-equivariant provers, correspond precisely to \emph{natural transformations} $\mathcal{I}\to\mathcal{P}$, which we properly introduce and elaborate on in \Cref{sec:proof-equivariance-naturality}.
Importantly, a prover satisfying \Cref{def:strong_consistency} assigns \emph{structurally matched} proofs matching probabilities, allowing for generalization between any pair of problems linked by an arrow in $\mathscr{R}$.

\textbf{Proof equivariance in the literature.}
Although Lean proof equivariance has not been formalized previously, some approaches share closely related inductive biases.
\draft{In particular, next-tactic predictors~\citep{yang2023leandojo, li2024hunyuanprover, xin2025bfsprover} operate on intermediate proof states rather than theorem statements, and can satisfy $\mathscr{R}_{\mathcal{T}}$-equivariance.
Note, however, that this symmetry breaks when these models rely on a search under fixed budgets (Appendix~\ref{app:state_tactic_equivariance}).
Moreover, their multi-step training introduces optimization challenges, often resulting in weaker performance than single-pass provers.}
At the same time, incorporating such bias into whole-proof generation models would require severe architectural changes.

Consequently, in this work, we focus on an adjacent notion of bias for formal provers, defined by the invariance of the success probability across statements that share the same mathematical content. 

\subsection{Success invariance}
\label{sec:semantic_consistency}

Proof $\mathscr{R}$-equivariance, stated in \Cref{sec:proof_equivariance}, requires proof distributions to transfer according to the arrows of a rewriting category $\mathscr{R}$.
With the examples from \Cref{fig:gcd_example} in mind, a less restrictive condition could require the prover to achieve the same success rate on problems \emph{equivalent} w.r.t $\mathscr{R}$:
\begin{definition}[\bf $\mathscr{R}$-equivalent statements]\label{def:r-equivalence}
    We say that statements $T, T' \in \mathscr{R}$ are \emph{$\mathscr{R}$-equivalent}, denoted as $T \sim_\mathscr{R} T'$, if $T\leftrightarrows T'$ in $\mathscr{R}$, i.e. there exist some arrows $\mathbf{t} : T\to T'$ and $\mathbf{t'} : T'\to T$.  
\end{definition}

The motivation behind \Cref{def:r-equivalence} is the ability to transform valid proofs between $T$ and $T'$ via the~arrows of $\mathscr{R}$.
We encapsulate the ability of models to perform such transfer by requiring that the~probabilities of success of solving $\mathscr{R}$-equivalent statements are equal:
\begin{definition}[\bf Success invariance]\label{def:weak_consistency}
A prover $p_\theta$ is \emph{$\mathscr{R}$-invariant in success probability} if
\[
s_\theta(T) \;=\; s_\theta(T')\qquad\text{for every }T\sim_{\mathscr{R}} T'.
\]
\end{definition}
Setting $\mathscr{R}=\mathscr{R}^*$ recovers the informal notion of success invariance.
An $\mathscr{R}$-invariant prover respects the transformations assigned by $\mathscr{R}$ and guarantees that $\mathscr{R}$-equivalent problems are correctly solved with the same probability.
Therefore, improvement on solving a specific theorem $T$ directly translates to improvements over the whole equivalence class $[T]_\mathscr{R}$.

\textbf{A hierarchy of invariances.}
Under 
$\mathscr{R}^*\supseteq\mathscr{R}_{\mathcal{T}}\supseteq\mathscr{R}$ of \Cref{sec:rewriting_categories}, we obtain the implications
\[
\underbrace{\mathscr{R}^*\text{-invariance}}_{\text{ideal success invariance}}
\;\Longrightarrow\;
\mathscr{R}_{\mathcal{T}}\text{-invariance}
\;\Longrightarrow\;
\mathscr{R}\text{-invariance}
\]
This hierarchy formalizes progressively weaker but more tractable notions of invariance. Invariance with respect to the unreachable maximal category $\mathscr{R}^*$ is the strongest, capturing what we ideally want from a prover.
Restricting to a tractable subcategory $\mathscr{R}\subseteq\mathscr{R}_{\mathcal{T}}$ is a necessary condition that we can computationally evaluate.
Our methodology pursues the latter, and the observed variance can be seen as a lower bound on the error of $\mathscr{R}^*$-invariance.

\vspace{-0.1em}
\section{Test-time scaling with rewriting for success invariance}
\label{sec:ensembles}
\vspace{-0.2em}

\Cref{sec:semantic_consistency} formalizes success invariance of a prover $p_\theta$ with respect to a rewriting category $\mathscr{R}\subseteq\mathscr{R}_{\mathcal{T}}$.
We now address how to realize this property in practice.
Prior work \citep{tian2025evolprover} encourages invariance through training-time data augmentation, but is limited to a small set of transformations and requires retraining.
Instead, we seek a \emph{test-time} method that, given an already-trained $p_\theta$ that may not satisfy success invariance, produces another prover~$\bar{p}_\theta$ that is (approximately) $\mathscr{R}$-invariant, without retraining~$p_\theta$. Our construction is inspired by sampling and canonicalization strategies for group invariance \citep{kim2024revisiting, kim2026inverting, kaba2023equivariance, schmidt2024tilt, shumaylov2024lie, singhal2025test}.

\subsection{The rewriting ensemble}

We take the rewriting subcategory $\mathscr{R}\subseteq\mathscr{R}_{\mathcal{T}}$ from \Cref{sec:rewriting_categories} as fixed, together with a sampling procedure $\mu(\cdot\mid T, k)$ that, given a theorem statement $T\in\mathsf{Stmt}$ and a number $k$, returns $k$ theorems $T_1, \dots, T_k\in\mathsf{Stmt}$ reachable from $T$ via a finite chain of arrows in $\mathscr{R}$.
These variants approximate the equivalence class $[T]_\mathscr{R}$, and serve as alternative representations of the same underlying problem.

\begin{definition}[\bf Rewriting ensemble]\label{def:rewriting_ensemble}
Fix a prover $p_\theta$, budget $N\in\mathbb{N}$ and a variant count $K \leq N$. The~\emph{$(K,N)$-rewriting ensemble prover} $\bar{p}_{\theta,\mu}^{(K,N)}$, applied to a statement $T$, proceeds as follows:
\vspace{-0.5em}
\begin{enumerate}[noitemsep,nolistsep,label=(\arabic*)]
\item Sample $K$ variants $\{T_1,\ldots,T_K\}\:\sim\:\mu(\cdot\mid T, K)$, with the corresponding transitions $T \xrightarrow{\mathbf{r}_i} T_i$.
\item Run $p_\theta$ for $\lfloor N/K\rfloor$ independent attempts on each $T_i$.
\item Return any resulting valid proof, verified with Lean, lifted to a proof of $T$ via the arrow $T \xrightarrow{\mathbf{r}_i} T_i$ in $\mathscr{R}$; return failure if no attempt succeeds.
\end{enumerate}
\end{definition}
\vspace{-0.5em}
Assuming independence of attempts, the expected success probability of the ensemble is
\[
\bar{s}_{\theta,\mu}^{(K,N)}(T) \;=\; 1 - \mathbb{E}_{\{T_1,\ldots,T_K\}\:\sim\:\mu(\cdot\mid T, K)}\!\left[\,\prod_{i=1}^K \bigl(1-s_\theta(T_i)\bigr)^{\lfloor N/K\rfloor}\,\right]
\]
Two limiting cases are worth bearing in mind: $K=1$ collapses the ensemble to running $p_\theta$ on a single random variant, and $K=N$ spreads the budget one attempt per variant. Our analysis below concerns how $\bar{s}_{\theta,\mu}^{(K,N)}(T)$ depends on $K$, and how it compares across $T\sim_{\mathscr{R}}T'$.

\subsection{Invariance in the sample-size limit}
\label{sec:sample-limit-consistency}

For the ensemble prover to satisfy the $\mathscr{R}$-invariance of \Cref{def:weak_consistency}, the sampling distribution $\mu(\cdot\mid T, K)$ must exhibit behavior generalizable over the whole equivalence class $[T]_{\mathscr{R}}$.
This is not automatic, as following arrows from a statement $T$ can, in general, reach different statements than when starting in $T'\sim_{\mathscr{R}} T$.
We hence assume that $\mu$ has syntactically invariant support:

\begin{assumption}[\bf Sampling coverage]\label{ass:sampling_coverage} Let $T \sim_\mathscr{R} T'$ and let $T'' \in \mathscr{R}$ be a statement. Suppose that for some $K$, the probability $\mathbb{P}(T'' \!\in \mu(\cdot \mid T, K)) $ of sampling $T''$ is greater than $0$. Then: 
\[
\lim_{K\to\infty} \mathbb{P}(T'' \!\in \mu(\cdot \mid T, K)) = \lim_{K\to\infty} \mathbb{P}(T'' \!\in \mu(\cdot \mid T', K)) = 1
\]
\end{assumption}
\vspace{-0.3em}
In words, if $T''$ can be sampled for some input theorem $T$, as the number of samples tends to infinity, it will eventually be sampled for any member of the equivalence class $[T]_\mathscr{R}$.
\Cref{ass:sampling_coverage} holds in two cases of interest: (i) when $\mu(\cdot \mid T, K)$ eventually covers every state $T''$ reachable from $T$ in $\mathscr{R}$, and~(ii) when the sampler performs \emph{canonicalization} \citep{kaba2023equivariance} and outputs a fixed representative of $[T]_{\mathscr{R}}$.

\begin{proposition}[\bf Invariance in the sample-size limit]\label{prop:invariance_limit}
Let $n:\mathbb{N} \rightarrow \mathbb{N}$ be a function such that $\lfloor n(x)/x\rfloor\to\infty$ as $x \to\infty$. Then, under \Cref{ass:sampling_coverage}:
\vspace{-0.3em}
\[
\lim_{K\to\infty}\bar{s}_{\theta,\mu}^{(K,n(K))}(T) \;=\; \lim_{K\to\infty}\bar{s}_{\theta,\mu}^{(K,n(K))}(T')\qquad\text{for every }T\sim_{\mathscr{R}}T',
\]
\vspace{-0.3em}
so the limiting ensemble prover is $\mathscr{R}$-invariant in success probability (\Cref{def:weak_consistency}).
\end{proposition}

Informally, averaging over enough variants reachable from the full equivalence class washes out the dependence on the particular initial theorem statement.
The proof is provided in \Cref{sec:proof-invariance-limit}.

\subsection{Monotonicity in the number of rewritings}
\vspace{-0.5em}
We now characterize the optimal number of variants of the input statement $T$, for reciprocal $\mathscr{R}$, regardless of the choice of sampling protocol $\mu$.
Assuming an equal prior on $s_\theta(T')$ for all $T'\sim _\mathscr{R} T$, spreading the budget across more versions is never worse, in expectation, than committing to fewer:

\begin{proposition}[\bf Monotonicity]\label{prop:monotonicity}
Let $\mathscr{R}$ be reciprocal and $T\in\Stmt$ be a statement.
Assume that for any $T' \sim _\mathscr{R} T$, the success probability $s_\theta(T') \in [0,1]$ is an independent random variable sampled from some (prior) distribution $P_{[T]_\mathscr{R}}$. Then, for any budget $N$ and divisors $K\le K'$ of $N$,
\vspace{-0.3em}
\[
\bar{s}_{\theta,\mu}^{(K,N)}(T) \;\le\; \bar{s}_{\theta,\mu}^{(K',N)}(T),
\]
\vspace{-0.2em}
with strict inequality whenever $K\!<\! K'$ and $s_\theta$ is non-constant on the support of $\mu(\cdot\mid T, K) \subseteq [T]_{\mathscr{R}}$.
\end{proposition}

\begin{corollary}
In particular, $\bar{s}_{\theta,\mu}^{(K,N)}(T)\ge\bar{s}_{\theta,\mu}^{(1,N)}(T)$ for every $K\geq 1$ dividing $N$.
\end{corollary}
\vspace{-0.6em}
The proof of \Cref{prop:monotonicity} is given in \Cref{sec:proof-monotonicity}.
Note that whenever $s_\theta(T)>0$, repeated sampling \citep{brown2024large} without rewriting also achieves $\bar{s} \to 1$ as $N \to \infty$, a property the ensemble does not improve on asymptotically.
The value of the rewriting ensemble is therefore at a finite budget $N$. Our empirical results in \Cref{sec:results} instantiate this comparison at budgets of practical interest.
\vspace{-0.4em}
\section{Methodology}
\label{sec:methodology}
\vspace{-0.3em}
Motivated by the sensitivity of LLM-based formal provers to specific formulations of theorems, and the guarantees from \Cref{sec:ensembles},
we design our experiments to study the following questions:
\begin{enumerate}[label=\textbf{(Q\arabic*)},leftmargin=30pt,noitemsep,nolistsep]
    \item How far from success invariance are the modern LLM-based single-pass formal provers?
    \item Can rewriting ensembles improve their robustness, and even increase performance? 
\end{enumerate}

We introduce the \emph{miniF2F-rw} dataset, consisting of verified, $\mathscr{R}_\mathcal{T}$-equivalent rewritings of problems from miniF2F \citep{zheng2022miniff}
and use it to test the robustness of leading open-source models \citep{ren2025deepseek,lin2025goedel,lin2025goedel_v2,kimina_prover_2025}.
We then confirm the beneficial impact of rewriting ensembles in two settings: \emph{controlled}, using the verified variants from miniF2F-rw, and \emph{test-time}, producing reformulations at inference time.

\draft{\textbf{Online augmentations.}
Augmenting formal statements is substantially harder than query rewriting in natural language tasks \citep{ma2023query,zhou2023leasttomost,obrien2024improving} due to strict syntactic constraints (Appendix \ref{sec:challenges_of_statement_reformulation}). To address this, we devise a three-stage test-time rewriting algorithm, outlined below and detailed in Appendix \ref{sec:sampling-mechanism}:}
\begin{enumerate}[noitemsep,nolistsep,label=(\arabic*),leftmargin=30pt]
    \item \textbf{Sampling.} Rewrites are drawn from the rewriting category $\mathscr{R}$ generated by non-conditional lemmas from Mathlib \citep{mathlib2020} supplemented with a small set of simplification tactics. 
    \item \textbf{Scoring.} Each sampled variant is ranked by a model-based energy function \citep{kim2026inverting}, based on the model's surprise when processing the input theorem, which acts as a cheap proxy for $s_\theta$.
    \item \textbf{Selection.} The top-scoring variants, together with the seed, are passed to the prover.
\end{enumerate}

\textbf{Benchmarks.}
We propose miniF2F-rw benchmark, built on top of miniF2F~\citep{zheng2022miniff}.
Each of the $488$ original problems is paired with $5$-$15$ rewrites, each with a Lean certificate of semantic equivalence, yielding roughly $5700$ statements in total.
All variants were additionally manually filtered to remove ones that changed mathematical content.
For more details, see Appendix~\ref{sec:minif2f-generation}.
To further demonstrate the potential of test-time ensembles, we use ProofNet \citep{azerbayev2023proofnet}, and Ineq-Comp \citep{zhao2025ineq} containing a `trans' split, designed specifically to test robustness to transformations of its `seed' problems.

\textbf{Models.}
Following \citet{tian2025evolprover}, we focus on leading open-source non-reasoning models in our main experiments: DeepSeek-Prover-V2 \citep{ren2025deepseek}, Goedel-Prover-DPO \cite{lin2025goedel}, and Goedel-Prover-SFT \citep{lin2025goedel}.
Reasoning provers produce thousands of tokens per proof \citep[Table 3]{ren2025deepseek}, which is prohibitive here.
We include smaller-scale evaluations of the robustness of chosen reasoning models in Appendix \ref{app:reasoning_experiments}.

\textbf{Evaluation.}
On miniF2F-rw, we compare each prover in four configurations, under fixed budget $k$:
\begin{itemize}[noitemsep,nolistsep,nosep]
    \item \textbf{seed}: all $k$ attempts on the original statement (corresponds to miniF2F benchmark performance).
    \item \textbf{random}: all $k$ attempts on a single variant, drawn uniformly from the pool in miniF2F-rw.
    \item \textbf{controlled}: $8$ variants drawn uniformly from the pool; budget split to $\frac{k}{8}$ attempts per variant.
    \item \textbf{test-time}: seed + $7$ variants sampled online; budget split to $\frac{k}{8}$ attempts per variant.
\end{itemize}

We report PASS@$k$ at three budgets, $k \in \{8, 32, 64\}$.
Each variant is attempted $n=64$ times.
As the results for \textbf{random} and \textbf{controlled} depend on the drawn variants, we sample 20,000 variant selections and report the mean and standard deviation.
An ablation on the ensemble size is in Appendix \ref{app:ablation_number_of_ensembles}.

We further evaluate \textbf{seed} and \textbf{test-time} mode for chosen provers on ProofNet \citep{azerbayev2023proofnet} and Ineq-Comp~\citep{zhao2025ineq}.
For the latter, in addition to PASS@$k$, we also measure the ratio of PASS@$k$ results on the `trans' and `seed' splits, displaying provers' generalizability to problems with similar proof structures \citep{zhao2025ineq, tian2025evolprover}.
We compare their results against EvolProver \cite{tian2025evolprover}, which proposed training data augmentations for implicitly learning generalization over equivalent formulations during training time.
\vspace{-0.5em}
\section{Evaluation results}
\label{sec:results}
\vspace{-0.5em}

\begin{table}
    \caption{PASS@$k$ results of selected LLM provers on the miniF2F-rw benchmark.
    The best result for each prover on each split is highlighted as \textbf{bold}, and the second-best is \underline{underlined}.}
    \label{tab:minif2f-rw-results}
    \scriptsize
    \centering
    \setlength{\tabcolsep}{3pt}
    
\begin{tabular}{ll@{\hspace{12pt}}lll@{\hspace{12pt}}lll}
\toprule
 & & \multicolumn{3}{c}{\textbf{miniF2F-rw-valid}} & \multicolumn{3}{c}{\textbf{miniF2F-rw-test}} \\
\midrule
\textbf{Model} & \textbf{Variant} & $k=8$ & $k=32$ & $k=64$ & $k=8$ & $k=32$ & $k=64$ \\
\midrule
\multirow{4}{*}{Goedel-Prover-SFT \cite{lin2025goedel}}
& seed & $\bm{57.7}$ & $\bm{62.8}$ & $\bm{64.8}$ & $\bm{51.7}$ & $\bm{56.5}$ & $\bm{57.8}$ \\
& random & $51.6{\scriptstyle \pm0.9}$ & $58.2{\scriptstyle \pm0.9}$ & $60.9{\scriptstyle \pm1.1}$ & $47.4{\scriptstyle \pm0.8}$ & $53.3{\scriptstyle \pm0.9}$ & $55.2{\scriptstyle \pm1.0}$ \\
& controlled & $54.7{\scriptstyle \pm0.2}$ & $60.8{\scriptstyle \pm0.2}$ & $63.4{\scriptstyle \pm0.2}$ & $50.0{\scriptstyle \pm0.1}$ & $55.7{\scriptstyle \pm0.1}$ & $57.4{\scriptstyle \pm0.2}$ \\
& test-time & $\underline{56.3}$ & $\underline{62.6}$ & $\underline{63.6}$ & $\underline{50.5}$ & $\underline{56.0}$ & $\underline{57.6}$ \\

\midrule
\multirow{4}{*}{Goedel-Prover-DPO \cite{lin2025goedel}}
& seed & $\bm{58.6}$ & $63.1$ & $64.3$ & $\bm{52.9}$ & $\underline{57.7}$ & $\underline{59.0}$ \\
& random & $53.9{\scriptstyle \pm1.0}$ & $59.7{\scriptstyle \pm1.0}$ & $61.6{\scriptstyle \pm1.2}$ & $49.5{\scriptstyle \pm0.9}$ & $54.7{\scriptstyle \pm1.0}$ & $56.4{\scriptstyle \pm1.1}$ \\
& controlled & $\underline{57.7{\scriptstyle \pm0.2}}$ & $\underline{63.4{\scriptstyle \pm0.2}}$ & $\bm{65.4{\scriptstyle \pm0.3}}$ & $\underline{52.6{\scriptstyle \pm0.2}}$ & $\bm{57.9{\scriptstyle \pm0.2}}$ & $\bm{59.9{\scriptstyle \pm0.2}}$ \\
& test-time & $\underline{57.7}$ & $\bm{63.5}$ & $\underline{65.2}$ & $49.1$ & $55.5$ & $57.1$ \\

\midrule
\multirow{4}{*}{DeepSeek-Prover-V2-7B \cite{ren2025deepseek}}
& seed & $\underline{78.4}$ & $79.1$ & $79.5$ & $61.7$ & $64.6$ & $66.8$ \\
& random & $71.1{\scriptstyle \pm1.1}$ & $74.1{\scriptstyle \pm1.1}$ & $75.3{\scriptstyle \pm1.2}$ & $59.7{\scriptstyle \pm0.9}$ & $63.2{\scriptstyle \pm1.0}$ & $64.8{\scriptstyle \pm1.1}$ \\
& controlled & $77.2{\scriptstyle \pm0.3}$ & $\underline{79.7{\scriptstyle \pm0.3}}$ & $\underline{80.8{\scriptstyle \pm0.4}}$ & $\bm{64.0{\scriptstyle \pm0.3}}$ & $\bm{67.3{\scriptstyle \pm0.3}}$ & $\bm{68.8{\scriptstyle \pm0.3}}$  \\
& test-time & $\bm{79.7}$ & $\bm{81.1}$ & $\bm{81.5}$ & $\underline{62.6}$ & $\underline{66.0}$ & $\underline{67.6}$ \\
\bottomrule
\end{tabular}

\end{table}
\begin{table*}[t]
    \vspace{-0.5em}
    \caption{The comparison between test-time ensembles and baselines on the considered datasets.}
    \label{tab:test-time-results-merged}
    \vspace{0.5em}
    \centering
    \tiny
    \setlength{\tabcolsep}{4pt}
    \begin{tabular}{lcccccc}
    \toprule
    \textbf{Model} & \textbf{Budget} & \textbf{ProofNet-test} & \textbf{Ineq-Comp-seed} & \textbf{Ineq-Comp-trans} & \textbf{Ineq-Comp-ratio} \\
    \midrule
    EvolProver \cite{tian2025evolprover} & 32 &  -- & $52.2$ & $34.0$ & $65.2$ \\  
    \midrule
    \multirow{2}{*}{Goedel-Prover-DPO \cite{lin2025goedel}}
        & 32 &  $13.6$ & $40.5$ & $12.3$ & $30.4\%$ \\
        & 64 &  $14.5$ & $44.0$ & $15.3$ & $34.8\%$ \\
    \midrule
    \multirow{2}{*}{Goedel-Prover-DPO + Ensemble}
        & 32 &  $15.6$ & $42.3$ & $14.5$ & $34.3\%$ \\
        & 64 &  $16.8$ & $46.6$ & $17.0$ & $36.5\%$ \\
    \midrule
    \multirow{2}{*}{DeepSeek-Prover-V2-7B \cite{ren2025deepseek}}
        & 32 &  $21.9$ & $64.8$ & $34.8$ & $53.7\%$ \\
        & 64 &  $23.1$ & $\bm{66.7}$ & $37.3$ & $55.9\%$ \\
    \midrule
    \multirow{2}{*}{DeepSeek-Prover-V2-7B + Ensemble}
        & 32 &  $21.9$ & $63.4$ & $40.8$ & $64.4\%$ \\
        & 64 &  $\bm{23.2}$ & $64.7$ & $\bm{45.5}$ & $\bm{70.3}\%$\\
    \bottomrule
    \end{tabular}
    \vspace{-1.5em}
\end{table*}
\begin{figure}[t]
    \vspace{-2em}
    \centering
%
%
%

\definecolor{exMini}{HTML}{6F4DB0}       
\definecolor{exProof}{HTML}{3567BC}      
\definecolor{exIneq}{HTML}{358352}       
\definecolor{exDivider}{HTML}{C8CFD9}    
\definecolor{exSucc}{HTML}{2A9C46}       
\definecolor{exFail}{HTML}{B8332E}       
\definecolor{exTacLabel}{HTML}{4F5763}   

\newcommand{\exBadge}[2]{%
    \tikz[baseline=(b.base)]{%
        \node[
            draw=#1!75, fill=#1!10, rounded corners=4pt, line width=0.4pt,
            inner xsep=4pt, inner ysep=1.2pt,
            font=\sffamily\fontsize{6.5}{7.5}\selectfont\bfseries,
            text=#1!90!black
        ] (b) {#2};%
    }%
}

\newcommand{\codefont}{\ttfamily\fontsize{6.5}{7.5}\selectfont}

\newcommand{\exBody}[4]{%
    \begin{minipage}[t][16mm][t]{\linewidth}%
        {\sffamily\tiny\bfseries\color{#1}#2}\hfill #3\\[3pt]%
        {\codefont #4}%
    \end{minipage}%
}

\newcommand{\exTacInner}[2]{%
    \begin{minipage}[t][5mm][t]{\linewidth}%
        {\color{#1}\rule[0.1ex]{4pt}{4pt}}\hspace{4pt}%
        \vspace{-1pt}
        {\sffamily\tiny\bfseries\color{exTacLabel}Tactic used:}\\%
        {\codefont #2}%
    \end{minipage}%
}

\newcommand{\exHeadInner}[2]{%
    \makebox[\linewidth][c]{%
        \rule[-1pt]{0pt}{11pt}%
        \sffamily\bfseries\small\color{#1}#2%
    }%
}

\newcommand{\exDashRule}{%
    \par\addvspace{3pt}%
    \noindent\hbox to \linewidth{%
        \color{exDivider}%
        \xleaders\hbox to 4pt{\hss\vrule height 0.4pt depth 0pt width 2.5pt\hss}\hfill%
    }%
    \par\addvspace{-5pt}%
}

\resizebox{\linewidth}{!}{%
\begin{tikzpicture}[
    font=\small,
    wMid/.style={text width=0.293\linewidth},  
    wNar/.style={text width=0.260\linewidth},  
    wWid/.style={text width=0.435\linewidth},  
    chipBase/.style={
        rounded corners=4pt, line width=0.5pt,
        inner xsep=6pt, inner ysep=4pt,
        align=center, anchor=north west,
    },
    miniHdr/.style ={chipBase, draw=exMini!55,  fill=exMini!13},
    proofHdr/.style={chipBase, draw=exProof!55, fill=exProof!13},
    ineqHdr/.style ={chipBase, draw=exIneq!55,  fill=exIneq!13},
    cardBase/.style={
        rounded corners=4pt, line width=0.5pt,
        fill=white, inner sep=5pt, align=left,
        anchor=north west,
    },
    miniCard/.style ={cardBase, draw=exMini!55},
    proofCard/.style={cardBase, draw=exProof!55},
    ineqCard/.style ={cardBase, draw=exIneq!55},
    tacBase/.style={
        rounded corners=4pt, line width=0.4pt,
        inner sep=5pt, align=left, anchor=north west,
    },
    miniTac/.style ={tacBase, draw=exMini!55,  fill=white},
    proofTac/.style={tacBase, draw=exProof!55, fill=white},
    ineqTac/.style ={tacBase, draw=exIneq!55,  fill=white},
]
    \node[miniHdr, wMid] (mh) at (0,0)
        {\exHeadInner{exMini}{miniF2F}};
    \node[miniCard, wMid, below=2pt of mh.south, anchor=north] (mc) {%
        \exBody{exMini}{Seed}{\exBadge{exFail}{0/64 SR}}{%
            \lkw{theorem} mathd\_algebra\_362 (a b : \lty{$\mathbb{R}$})\\[-0.35em]
            \hspace*{1em}(h$_0$ : a \textasciicircum\: 2 * b \textasciicircum\: 3 = 32 / 27)\\[-0.35em]
            \hspace*{1em}(h$_1$ : a / b \textasciicircum\: 3 = 27 / 4)\\[-0.35em]
            \hspace*{1em}: a + b = 8 / 3 \lkw{:= by}}%
        \exDashRule
        \exBody{exMini}{Augmented}{\exBadge{exSucc}{2/16 SR}}{%
            \lkw{theorem} mathd\_algebra\_362 (a b : \lty{$\mathbb{R}$})\\[-0.35em]
            \hspace*{1em}(h$_0$ : a \textasciicircum\: 2 * b \textasciicircum\: 3 = 32 / 27)\\[-0.35em]
            \hspace*{1em}(h$_1$ : b \textasciicircum\: 3 / a = 4 / 27)\\[-0.35em]
            \hspace*{1em}: a + b = 8 / 3 \lkw{:= by}}%
    };
    \node[miniTac, wMid, below=2pt of mc.south west, anchor=north west] (mt)
        {\exTacInner{exMini}{\ltac{rw} [div\_eq\_div\_iff\_comm] \lkw{at} h$_1$}};

    \node[proofHdr, wNar, right=1pt of mh.north east, anchor=north west] (ph)
        {\exHeadInner{exProof}{ProofNet}};
    \node[proofCard, wNar, below=2pt of ph.south, anchor=north] (pc) {%
        \exBody{exProof}{Seed}{\exBadge{exFail}{0/64 SR}}{%
            \lkw{theorem} exercise\_1\_1\_17\\[-0.35em]
            \hspace*{1em}\{G : \lty{Type*}\} [\lty{Group} G] \{x : G\}\\[-0.35em]
            \hspace*{1em}\{n : \lty{$\mathbb{N}$}\} (hxn : orderOf x = n)\\[-0.35em]
            \hspace*{1em}: x\textasciicircum{-1} = x\textasciicircum (n-1 : \lty{$\mathbb{Z}$}) \lkw{:= by}}%
        \exDashRule
        \exBody{exProof}{Augmented}{\exBadge{exSucc}{16/16 SR}}{%
            \lkw{theorem} exercise\_1\_1\_17\\[-0.35em]
            \hspace*{1em}\{G : \lty{Type*}\} [\lty{Group} G] \{x : G\}\\[-0.35em]
            \hspace*{1em}\{n : \lty{$\mathbb{N}$}\} (hxn : orderOf x = n)\\[-0.35em]
            \hspace*{1em}: x\textasciicircum{-1} = x\textasciicircum n * x\textasciicircum{-1} \lkw{:= by}}%
    };
    \node[proofTac, wNar, below=2pt of pc.south west, anchor=north west] (pt)
        {\exTacInner{exProof}{\ltac{rw} [zpow\_sub\_one, zpow\_coe\_nat]}};

    \node[ineqHdr, wWid, right=1pt of ph.north east, anchor=north west] (ih)
        {\exHeadInner{exIneq}{Ineq-Comp}};
    \node[ineqCard, wWid, below=2pt of ih.south, anchor=north] (ic) {%
        \exBody{exIneq}{Seed}{\exBadge{exFail}{0/64 SR}}{%
            \lkw{theorem} cauchy\_p19 (x y z : \lty{$\mathbb{R}$})\\[-0.35em]
            \hspace*{1em}(hx : x > 0) (hy : y > 0) (hz : z > 0)\\[-0.35em]
            \hspace*{1em}(h : 1/(1 + x\textasciicircum 2) + 1/(1 + y\textasciicircum 2) + 1/(1 + z\textasciicircum 2) = 2)\\[-0.35em]
            \hspace*{1em}: x\textasciicircum 2 + y\textasciicircum 2 + z\textasciicircum 2 + 3 $\ge$ (x + y + z)\textasciicircum 2 \lkw{:= by}}%
        \exDashRule
        \exBody{exIneq}{Augmented}{\exBadge{exSucc}{8/16 SR}}{%
            \lkw{theorem} cauchy\_p19 (x y z : \lty{$\mathbb{R}$})\\[-0.35em]
            \hspace*{1em}(hx : 0 < x) (hy : 0 < y) (hz : 0 < z)\\[-0.35em]
            \hspace*{1em}(h : (1 + x\textasciicircum 2)$^{-1}$ + (1 + y\textasciicircum 2)$^{-1}$ + (1 + z\textasciicircum 2)$^{-1}$ = 2)\\[-0.35em]
            \hspace*{1em}: (x + y + z)\textasciicircum 2 $\le$ x\textasciicircum 2 + y\textasciicircum 2 + z\textasciicircum 2 + 3 \lkw{:= by}}%
    };
    \node[ineqTac, wWid, below=2pt of ic.south west, anchor=north west] (it)
        {\exTacInner{exIneq}{\ltac{norm\_num} \lkw{at} *}};
\end{tikzpicture}%
}
\vspace{-1em}
    \vspace{-0.5em}
    \caption{Examples of problems solved by the test-time ensemble framework, where seed failed. }
    \label{fig:examples_of_solved_rewrites}
    \vspace{-0.7em}
\end{figure}

\textbf{Success invariance.} 
Results on miniF2F-rw (\Cref{tab:minif2f-rw-results}) reveal a consistent drop in performance under \textbf{random} perturbations relative to the \textbf{seed} baseline \emph{for all tested models}, confirming their sensitivity to input formulation.
Interestingly, this degradation is substantially larger on the validation split.
For example, at PASS@$8$, DeepSeek-Prover-V2 drops by $7.3$ points on the validation split, compared to only $2.0$ on test.
This reveals its dependence on surface structures witnessed during model selection.

Rewriting ensembles effectively mitigate this gap.
Both \textbf{controlled} and \textbf{test-time} variants significantly outperform single random rewrites (e.g., +$4\%$ PASS@64 for DeepSeek on test), while additionally reducing variance.
Notably, despite their bias toward original statements, both Goedel-Prover-DPO and DeepSeek-Prover-V2 benefit from ensembling, surpassing the \textbf{seed} baseline.
This empirically validates the theoretical advantages outlined in \Cref{sec:ensembles}.
Remarkably, DeepSeek-Prover-V2's ensemble performance is competitive with or surpasses reasoning at matched token counts (Appendix~\ref{app:minif2f-tokens}).

\textbf{Test-time ensembling.}
We compare \textbf{seed} baselines and \textbf{test-time} ensembles across benchmarks (\Cref{tab:test-time-results-merged}), using EvolProver as a retraining-based reference. Despite models' bias toward seed formulations in ProofNet-test and Ineq-Comp-seed, test-time ensembling consistently yields competitive or improved performance over the baselines.
On the main split of interest, Ineq-Comp-trans, rewriting ensemble increases PASS@64 of DeepSeek-Prover-V2 from 37.3\% to 45.5\%, establishing a new state of the art among non-reasoning models.
The resulting robustness ratio is competitive with EvolProver ($65.2\%$ vs $64.4\%$ at PASS@$32$).
This shows that our framework can efficiently reduce sensitivity of underlying provers to surface forms of problems, without retraining.

\draft{Across all datasets, we observe cases where DeepSeek-Prover fails on the original formulation yet consistently solves rewritten variants (\Cref{fig:examples_of_solved_rewrites}).
In several examples, rewriting increases the success rate from $0/64$ on seed to $8/16$, or even $16/16$, on the reformulation.
These results show that our method discovers non-trivial rewrites that substantially improve solvability across diverse domains.}
\vspace{-0.5em}
\section{Related work}
\label{sec:related_work}
\vspace{-0.5em}

\textbf{Statement augmentations.}
To evaluate the generalization of Lean provers across related problems, \citet{zhao2025ineq} introduces Ineq-Comp, a benchmark of inequality questions and their algebraic transformations.
While not strictly equivalent under our definition, these variants share the proof structure, hence testing compositional reasoning.
To address the observed robustness gap, EvolProver \citep{tian2025evolprover} augments training data with fixed abstract syntax tree (AST) transformations, encoding common symmetries. Our approach improves consistency through test-time augmentation, without retraining.
FormalEvolve \citep{formalevolve} showed that enforcing the generation of diverse reformulations of formal statements benefits autoformalization \cite{zhang2026drift, cabral2026proofflow}; we extend this insight to LLM-based proof generation.

\textbf{Equivalence of formal statements.}
Determining equivalence between formal theorems is a core challenge in autoformalization, the task of mapping natural language problems to formal representations.
Reliable evaluation requires accurately verifying equivalence with the ground truth.
BEq \citep{liu2025beq}, defines equivalence via reducibility: $T$ is equivalent to $T'$ if $T$ can be proved~using~$T'$, and vice versa.
In contrast, we adopt a perspective based on co-reachability in rewriting categories.

\textbf{Structural representations in formal theorem proving.}
Prior work has explored explicit structural encodings of formal statements \citep{paliwal2020graph,li2021neurotactic,blaauwbroek2024graph2tac}, yet typically representing them as AST-based graphs processed with graph neural networks.
While these approaches naturally capture symmetries such as variable renaming and commutativity, they have struggled to scale to harder problems compared to LLM-based methods.
To bridge this gap, Algebraic Positional Encodings (APE) \citep{kokos2024representations} incorporate algebraic structure, e.g., of binary trees, directly into the attention mechanisms. In \cite{kokos2024learning}, they were applied to theorem proving by encoding inputs as ASTs and training models to exploit their structure.

\textbf{Symmetries in deep learning.}
Symmetry is usually modeled as group actions~\citep{bronstein2021geometric} and incorporated into neural networks either via weight sharing \citep{pmlr-v70-ravanbakhsh17a, kondor2018generalizationequivarianceconvolutionneural}, data augmentation \citep{chen2020grouptheoreticframeworkdataaugmentation, lyle2020benefitsinvarianceneuralnetworks}, or test-time methods \citep{kim2024revisiting, kim2026inverting, kaba2023equivariance, schmidt2024tilt, shumaylov2024lie, singhal2025test}.
Prior works generalized to categories via monad algebras \citep{gavranovic2024position} or naturality \citep{maruyama2025categorical}, although focusing on theorizing weight sharing.
\citet{bottou2024conceptual} posited that rewrite rules of natural language~\citep{harris1968mathematical, harris1991theory} may form a groupoid, a category with all arrows invertible, though not necessarily preserving semantics.
We extend these views, modeling semantics-preserving rewrites with a category, using naturality to formulate invariance and equivariance, and realizing invariance at test time.
\vspace{-0.5em}
\section{Summary, limitations, and future work}
\vspace{-0.5em}
\label{sec:summary}
We frame Lean theorem proving in a category-theoretic framework, modeling statement augmentations as arrows in rewriting categories.
This perspective yields two symmetry principles for LLM-based provers: success invariance and proof equivariance.
We show that state-of-the-art provers violate these properties, exhibiting strong sensitivity to equivalent reformulations.
To address this, we propose a test-time ensembling method that approximates invariance for arbitrary provers.
We provide theoretical guarantees and show improved robustness and performance across benchmarks.

\draft{We restrict our experiments with ensembles to non-reasoning provers.
Although we expect the qualitative picture to carry over to reasoning models,
whether the \emph{magnitude} of the achieved marginal gains differs is left as an open question.
Another natural extension is to increase the ensemble expressivity, which currently uses a model-based energy function as a lightweight proxy for~$s_\theta$.
A~selector trained to predict proof success, or an adaptive bandit allocating compute across variants at inference, are both possible next steps.
A further idea is to incorporate our symmetry framework into model design, e.g., by extending APE \cite{kokos2024representations} to categorical structures.
Lastly, outside the mathematical domain, rewriting categories could be used to model symmetries of formal languages in programming and logic \citep{pei2023exploiting, chang2026dynamics} (see \Cref{app:formal_languages}), as well as certain aspects of natural language \citep{irpan2025consistency, berglund2023reversal, kusner2017counterfactual, presposition}.}
\section*{Acknowledgments}

The authors would like to acknowledge the use of the University of Oxford Advanced Research Computing (ARC) facility in carrying out this work. \href{https://doi.org/10.5281/zenodo.22558}{https://doi.org/10.5281/zenodo.22558}.
We also would also like to thank Rob Cornish for helpful discussions.

\bibliographystyle{unsrtnat}
\bibliography{references}


\appendix

\section{Background on category theory}
\label[appendix]{app:extended_category}

Our formalization uses a small amount of elementary category theory.
We leave a minimal introduction here, and refer the reader to \citet{leinster2014basic} for further background.

\subsection{Categories}
Categories provide a natural abstraction of local, non-invertible transformations, just as groups provide a natural abstraction of context-independent (i.e.\ global) and invertible symmetries.

\textbf{\Cref{def:category}}
A \emph{category} $\mathscr{C}$ consists of
\begin{itemize}[noitemsep,nolistsep]
\item a collection of \emph{objects} $\ob(\mathscr{C})$,
\item for each $A,B\in\ob(\mathscr{C})$, a collection $\mathscr{C}(A,B)$ of \emph{maps} or \emph{arrows} from $A$ to $B$,
\item for each pair of arrows $f\in\mathscr{C}(A,B)$ and $g\in \mathscr{C}(B,C)$, their \emph{composition} $g\circ f\in\mathscr{C}(A,C)$,
\item for each object $A\in \ob(\mathscr{C})$, an element $1_A$ of $\mathscr{C}(A,A)$, called the \emph{identity} on $A$,
\end{itemize}
satisfying the following axioms:
\begin{itemize}[noitemsep,nolistsep]
\item \emph{associativity}: $(h\circ g)\circ f= h\circ(g\circ f)$ for any $f\in \mathscr{C}(A,B)$, $g\in \mathscr{C}(B,C)$ and $h\in \mathscr{C}(C,D)$,
\item \emph{identity laws}: for each $f\in \mathscr{C}(A,B)$, we have $f\circ1_A = f = 1_B\circ f$.
\end{itemize}

\textbf{\Cref{example:category_of_sets}}
There is a category $\mathbf{Set}$ whose objects are sets, whose maps are ordinary functions, with function composition as composition and identity functions as identities. 

Given \Cref{example:category_of_sets} above, it may be tempting to think of maps as functions, but the definition only requires that they compose associatively with identities:

\begin{example}\label{example:group_as_category}
Any group $G$ is a category $\mathbf{B}G$ with a single object $\bullet$ and arrows $\mathbf{B}G(\bullet,\bullet)=G$, where composition is the group operation and the identity arrow is the group identity. The invertibility of group elements translates to the statement that every arrow of $\mathbf{B}G$ is an isomorphism.
\end{example}

\subsection{Functors}

To relate categories to one another and, in particular, to let an abstract category act on concrete data, we use structure-preserving maps between categories, called \emph{functors}.

\textbf{\Cref{def:functor}}
A functor $F:\mathscr{C}\to\mathscr{D}$ consists of an object map $F:\ob(\mathscr{C})\to\ob(\mathscr{D})$, and, for each pair $A,B\in\ob(\mathscr{C})$, an assignment on arrows, $F:\mathscr{C}(A,B)\to\mathscr{D}(F(A),F(B))$,
satisfying:
\begin{itemize}[noitemsep,nolistsep]
\item \emph{preservation of composition}: $F(g\circ f)=F(g)\circ F(f)$ whenever $A\xrightarrow{f}B\xrightarrow{g}C$ in $\mathscr{C}$,
\item \emph{preservation of identities}: $F(1_A)=1_{F(A)}$ for every $A\in\mathscr{C}$.
\end{itemize}

The target $\mathscr{D}$ may differ substantially from the source $\mathscr{C}$, but the compositional structure of $\mathscr{C}$ is reflected in its image under $F$.
Of particular importance in our setting are functors into $\mathbf{Set}$, which model how the abstract transformations in $\mathscr{C}$ \emph{act on} concrete data.
We illustrate this with groups:
\begin{example}\label{example:group_action_as_functor}
A group action of $G$ on a set $X$ is exactly a functor $F:\mathbf{B}G\to\mathbf{Set}$ sending the unique object $\bullet$ to $X$ and each group element $g\in G$ to a function $F(g):X\to X$.
The functor axioms unfold to the familiar identities $F(gh)=F(g)\circ F(h)$ and $F(1)=1_X$, which are precisely the defining axioms of a group action.
\end{example}

We will also find useful a simple \emph{implementation functor}, mapping theorems from rewriting categories to the corresponding singleton sets:

\textbf{\Cref{example:implementation-functor}.}
    There is an implementation functor $\mathcal{I} : \RT \to \Set$ mapping each theorem $T\in \Stmt$ to the singleton set $\{T\}$, while arrows are mapped to the unique functions between singletons.

\subsection{Natural transformations}

We now introduce natural transformations, the categorical analogue of invariant and equivariant maps.
Just as an equivariant map is defined between two $G$-sets (sets equipped with an action of a group $G$), a natural transformation is defined between two functors sharing a common source and target.
\begin{definition}\label{def:natural_transformation}
Let $F,G:\mathscr{C}\to\mathscr{D}$ be two functors. A natural transformation $\alpha:F\to G$ is an object-indexed family of arrows $(\alpha_A:F(A)\to G(A))_{A\in\ob(\mathscr{C})}$ in $\mathscr{D}$ such that, for every arrow $f:A\to B$ in $\mathscr{C}$, the following \emph{naturality square} commutes:
\[
\begin{tikzcd}
F(A) \arrow[r, "F(f)"] \arrow[d, "\alpha_A"'] & F(B) \arrow[d, "\alpha_B"] \\
G(A) \arrow[r, "G(f)"'] & G(B)
\end{tikzcd}
\]
Equivalently, $\alpha_B\circ F(f)=G(f)\circ\alpha_A$ for every arrow $f:A\to B$.
\end{definition}
The commutativity of this square, called \emph{naturality}, says that the family $\alpha$ is compatible with the transformations encoded by $\mathscr{C}$: transporting along $\mathscr{C}$ and then applying $\alpha$ yields the same result as applying $\alpha$ and then transporting along the corresponding arrow of $\mathscr{D}$.
This compatibility is exactly the content of invariance and equivariance in our setting.
\begin{example}\label{example:group_equivariance}
Taking $\mathscr{C}=\mathbf{B}G$, a natural transformation $\alpha:F\to H$ between two group actions $F,H:\mathbf{B}G\to\mathbf{Set}$ reduces to a single map $\alpha_\bullet:F(\bullet)\to H(\bullet)$ whose naturality square, instantiated at each $g\in G$, becomes the familiar equivariance condition $\alpha_\bullet(g\cdot x)=g\cdot\alpha_\bullet(x)$.
When $H$ acts trivially on the target (i.e.\ $H(g)=1\:\:\forall g\in G$), this reduces further to $\alpha(g\cdot x)=\alpha(x)$, recovering invariance.
\end{example}
The example also illustrates the global character of group transformations: because $\mathbf{B}G$ has a single object, a natural transformation collapses to one map $\alpha_\bullet$ that applies uniformly.
In a category with many objects, such as the rewriting categories (introduced in \Cref{sec:rewriting_categories}), each object $A$ carries its own component $\alpha_A$, and naturality is what coordinates them across the arrows of $\mathscr{C}$. This is how the framework captures local, context-dependent transformations central to our setting.

\section{Extended formalization of rewriting categories}

\subsection{Proof equivariance as natural transformations}
\label[appendix]{sec:proof-equivariance-naturality}

Let $\mathscr{R}$ be a rewriting category. Recall the definition of the proof functor on $\mathscr{R}$ from \Cref{sec:proof_equivariance}:

\textbf{\Cref{def:proof_functor}.}
The \emph{(stochastic) proof functor} on $\mathscr{R}$ is the functor $\mathcal{P}:\mathscr{R}\to\mathbf{Set}$ that maps each object $T\in \mathscr{R}$ to $\mathcal{P}(T) = \mathbb{D}_+(\mathcal{T}^*)$,
and whose action on arrows $\mathbf{t} : T \rightarrow T'$ is defined as:
\[
[\mathcal{P}(\mathbf{t})](\nu)(\mathbf{t}') \propto \nu([\mathbf{t}\!:\! \mathbf{t}'])\qquad \forall \nu \in \mathbb{D}_+(\mathcal{T}^*), \mathbf{t}'\in \mathcal{T}^* 
\]

We then formulate an inductive bias for formal theorem provers in terms of the proof functor $\mathcal{P}$, requiring that their proof distributions can be transported over arrows of $\mathscr{R}$ by the action of $\mathcal{P}$:

\textbf{\Cref{def:strong_consistency}.}
Let $\mathcal{P}:\mathscr{R}\to\mathbf{Set}$ be the proof functor.
A prover $p_\theta$ is \emph{$\mathscr{R}$-equivariant in proof distribution} (with respect to $\mathcal{P}$) if
\[
p_\theta\left(\cdot \mid T'\right) = \mathcal{P}(\mathbf{t})(p_\theta(\cdot \mid T))
\qquad\text{for every arrow }\mathbf{t}:T\to T'\text{ in }\mathscr{R}.
\]

For \Cref{def:strong_consistency} to be well-defined, as $\mathcal{P}(\mathbf{t})$ is a function $\DpT\!\to\!\DpT$, we need to assume that $p_\theta(\cdot\mid T)\in\DpT$ for every theorem statement $T\in\mathscr{R}$.
This is not a radical simplification, as whole-proof generation LLMs, which rely on logits for next token generation, satisfy this property.

\textbf{\Cref{ass:non-zero-proof-chance}.}
    For any $T\in\Stmt$, the prover $p_\theta$ satisfies $p_\theta(\cdot\mid T) \in \DpT$.

Let $\mathcal{I}$ be the {implementation functor} for $\mathscr{R}$.
We can treat $p_\theta$ as a family $((p_\theta)_T)_{T\in\mathscr{R}}$ of functions $(p_\theta)_T: \{T\} \to \DpT$ mapping $((p_\theta)_T)(T) = p_\theta(\cdot \mid T)$. With this perspective, let $\mathbf{t}: T \to T'$ be an arrow in $\mathscr{R}$. We can rewrite the sides of the equality from \Cref{def:strong_consistency} as:
\[
p_\theta(\cdot\mid T') = ((p_\theta)_{T'})(T') = ((p_\theta)_{T'})(\mathcal{I}(\mathbf{t})(T)) = ((p_\theta)_{T'} \circ \mathcal{I}(\mathbf{t}))(T)  
\]
Similarly,
\[
\mathcal{P}(\mathbf{t})(p_\theta(\cdot \mid T)) = \mathcal{P}(\mathbf{t})(((p_\theta)_T)(T)) = (\mathcal{P}(\mathbf{t})\circ (p_\theta)_T)(T)
\]
which allows us to rephrase the condition for $\mathscr{R}$-equivariance in proof distribution as:
\[
((p_\theta)_{T'} \circ \mathcal{I}(\mathbf{t}))(T) = (\mathcal{P}(\mathbf{t})\circ (p_\theta)_T)(T)
\qquad\text{for every arrow }\mathbf{t}:T\to T'\text{ in }\mathscr{R}.
\]
Both functions $(p_\theta)_{T'} \circ \mathcal{I}(\mathbf{t})$ and $\mathcal{P}(\mathbf{t})\circ (p_\theta)_T$ are of the type $\{T\}\to\DpT$, so them being equal on the argument $T$ translates to the equality of the functions themselves. As a result, we can simplify the requirement above into:
\[
((p_\theta)_{T'} \circ \mathcal{I}(\mathbf{t})) = (\mathcal{P}(\mathbf{t})\circ (p_\theta)_T)
\qquad\text{for every arrow }\mathbf{t}:T\to T'\text{ in }\mathscr{R}.
\]
But this is precisely the definition of $p_\theta$ being a natural transformation $\mathcal{I}\to \mathcal{P}$.

In conclusion, we have shown that both proposed notions of inductive bias for Lean-based formal theorem provers can be precisely characterized as naturality conditions, and hence as natural transformations.
This formulation extends the perspective of geometric deep learning \citep{bronstein2021geometric}, where structural inductive biases are captured via invariance and equivariance under group actions, to a categorical setting.
In this broader context, operating over functors $\mathscr{R} \to \Set$, the appropriate analogue of such equivariance (and also invariance, as we show in the next section) is given by naturality.
A similar formulation has been proposed by \citep{maruyama2025categorical}, but without being grounded in concrete applications.
We find that categories are a natural fit for modeling semantics-preserving rewritings in formal languages.

\subsection{Success invariance as natural transformations}
\label[appendix]{app:success_invariance}

Recall that given a rewriting category $\mathscr{R}$, we define the notion of prover $\mathscr{R}$-invariance as follows: 

\textbf{\Cref{def:r-equivalence}.}
    We say that statements $T, T' \in \mathscr{R}$ are \emph{$\mathscr{R}$-equivalent}, denoted as $T \sim_\mathscr{R} T'$, if $T\leftrightarrows T'$ in $\mathscr{R}$, i.e. there exist some arrows $\mathbf{t} : T\to T'$ and $\mathbf{t'} : T'\to T$.  

\textbf{\Cref{def:weak_consistency}.}
A prover $p_\theta$ is \emph{$\mathscr{R}$-invariant in success probability} if
\[
s_\theta(T) \;=\; s_\theta(T')\qquad\text{for every }T\sim_{\mathscr{R}} T'.
\]

For the reciprocal $\mathscr{R}$, we can equivalently formalize this definition in terms of natural transformations.

Let us introduce a functor $\mathcal{J}_\mathscr{R}:\mathscr{R} \to \Set$ that maps each object $T\in \mathscr{R}$ to the interval $[0,1]\subset \mathbb{R}$,
and every arrow to the identity function $1_{[0,1]}$.
Consider the implementation functor $\mathcal{I} : \mathscr{R} \to \Set$.

For any $T\in \mathscr{R}$, we have $\mathcal{I}(T) = \{T\}$ and $\mathcal{J}(T)=[0,1]$.
This creates a perfect setting to define the success probability map $s_\theta$ as a family of functions $\left(\mathcal{I}(T) \to \mathcal{J}(T)\right)_{T\in\mathscr{R}}$.
Concretely, for a given prover $p_\theta$, consider functions $\mathcal{S}^\theta_T : \{T\}\to[0,1]$ given by $\mathcal{S}_T^\theta(T) = s_\theta(T)$.
By the observation above, each of these functions is of type $\mathcal{S}_T^\theta : \mathcal{I}(T) \to \mathcal{J}(T)$.
This gives rise to a family $\mathcal{S}^\theta$ of functions $(\mathcal{S}_T^\theta : \{T\} \to [0,1])_{T\in \ob(\mathscr{R})}$.
For reciprocal $\mathscr{R}$, naturality of $\mathcal{S}^\theta$ coincides with $\mathscr{R}$-invariance of $p_\theta$: 

\begin{claim}
    When $\mathscr{R}$ is reciprocal, $\mathcal{S}^\theta$ is a natural transformation $\mathcal{I}\to\mathcal{J}$ if and only if the prover $p_\theta$ satisfies Definition \ref{def:weak_consistency}.
\end{claim}

\begin{proof}
    Recall that in a reciprocal category $\mathscr{R}$, $T \sim_\mathscr{R} T'$ if and only if there is some arrow $T \to T'$.
    Let $\mathbf{t} : T \to T'$ be an arrow in $\mathscr{R}$.  
    The naturality condition from \Cref{def:natural_transformation} requires:
    \[
    \mathcal{S}^\theta_{T'} \circ \mathcal{I}(\mathbf{t}) = \mathcal{J}(\mathbf{t})\circ \mathcal{S}^\theta_T
    \]
    Both these functions are from $\mathcal{I}(T)=\{T\}$ to $\mathcal{J}(T')=[0,1]$, so the functions above are equal if and only if they act equally on $T$.
    Finally, unfolding their definitions, we get:
    \[
    \begin{aligned}
    \left(\mathcal{S}^\theta_{T'} \circ \mathcal{I}(\mathbf{t})\right)(T) &= \mathcal{S}^\theta_{T'}(T') = s_\theta(T')\\
    \left(\mathcal{J}(\mathbf{t})\circ \mathcal{S}^\theta_T\right)(T) &= \mathcal{J}(\mathbf{t})(s_\theta(T)) = s_\theta(T)\: 
    \end{aligned}
    \]
    Combining the observations listed above, we can derive the following equivalence chain: 
    \[
    \begin{aligned}
        p_\theta \text{ is } \mathscr{R}\text{-invariant}
        &\iff s_\theta(T)=s_\theta(T') \quad\text{for all } T\sim_\mathscr{R} T' \\
        &\iff s_\theta(T)=s_\theta(T') \quad\text{whenever there is an edge } \mathbf{t} : T \to T'\\
        &\iff \left(\mathcal{S}^\theta_{T'} \circ \mathcal{I}(\mathbf{t})\right)(T) = \left(\mathcal{J}(\mathbf{t})\circ \mathcal{S}^\theta_T\right)(T) \quad\text{whenever}\ \mathbf{t} : T \to T' \text{ is in }\mathscr{R}\\
        &\iff \mathcal{S}^\theta_{T'} \circ \mathcal{I}(\mathbf{t}) = \mathcal{J}(\mathbf{t})\circ \mathcal{S}^\theta_T \quad\text{whenever } \mathbf{t} : T \to T' \text{ is an arrow in }\mathscr{R}\\
        &\iff \mathcal{S}^\theta \text{ is a natural transformation } \mathcal{I}\to\mathcal{J}
    \end{aligned}
    \]
\end{proof}

\begin{remark}
    More generally, even when $\mathscr{R}$ is not reciprocal, $\mathcal{S}^\theta$ defines a natural transformation $\mathcal{I}\to\mathcal{J}$ provided $s_\theta(T)=s_\theta(T')$ for any statements $T,T'$ connected by an arrow in $\mathscr{R}$.
    Although this condition does not imply equivalence, an appropriate choice of generating tactics for $\mathscr{R}$ can still induce useful structure in an $\mathscr{R}$-invariant prover.
    For instance, if the generating set includes tactics invoking external solvers (e.g., \leaninline{omega}), $\mathscr{R}$-invariance may correspond to the model’s ability to reduce statements to forms automatically closed by such tactics.
\end{remark}

\subsection{Structural biases of next-tactic prediction models}
\label[appendix]{app:state_tactic_equivariance}

In this work, we study whole-proof LLM-based provers, which take a Lean theorem statement as input and autoregressively generate a complete proof in a single pass.
This approach benefits from efficient inference in modern LLMs, but requires committing to an entire reasoning trajectory without intermediate feedback from the Lean compiler.

\draft{An alternative approach formulates theorem proving as search over Lean states, with tactics as actions. Next-tactic predictors~\citep{yang2023leandojo, li2024hunyuanprover, xin2025bfsprover} condition on the current proof state and predict the next tactic, hence defininig a conditional distribution $m_\theta(\cdot \mid \mathcal{L}(T))$.
They are typically integrated into best-first~\citep{li2024hunyuanprover, xin2025bfsprover} or Monte Carlo tree search~\citep{xin2025deepseekproverv} frameworks, transforming $m_\theta$ into a multi-step prover $M_\theta(\cdot\mid T)$.
From the perspective of inductive biases, of particular interest are \emph{sequential} multi-step provers, corresponding to the special case of beam size $1$.
We will call a multi-step prover $M_\theta$ \emph{sequential}, if it repeatedly samples a tactic, applies it to the current state, and continues from the resulting state.
The probability of generating a proof $\mathbf{t} = (t_1,\dots,t_n)\in\mathcal{T}^*$ is then}
\[
M_\theta(\mathbf{t} \mid T) = \prod_{i=1}^n m_\theta(t_i \mid \mathbf{t}_{<i}(\mathcal{L}(T)))
\]
By operating on states, these frameworks can implicitly capture the proposed inductive biases:

\textbf{$\RT$ proof equivariance.} Let $T\xrightarrow{\bf{t}} T'$ be an arrow in $\RT$. This corresponds to the relation between induced states: $\mathcal{L}(T') = \mathbf{t}(\mathcal{L}(T))$.
Let $\mathcal{P}:\RT\to\mathbf{Set}$ be the proof functor. Suppose $M_\theta$ is \draft{a sequential multi-step prover. Recall that we would call $M_\theta$ $\RT$-equivariant in proof distribution if:
\[
M_\theta\left(\cdot \mid T'\right) = \mathcal{P}(\mathbf{t})(M_\theta(\cdot \mid T))
\qquad\text{for every arrow }\mathbf{t}:T\to T'\text{ in }\mathscr{R}.
\]
Now, let us unwrap the substance of the right-hand side of the equation above.
$\mathcal{P}(\mathbf{t})$ is a function that takes a distribution over proofs $\mathcal{T}^*$ and returns the induced marginal distribution of proofs that start with $\bf t$.
Therefore, for any $\mathbf{t'}\in\mathcal{T}^*$, we have:
\[
\begin{aligned}
\mathcal{P}(\mathbf{t})(M_\theta(\cdot \mid T))(\mathbf{t'}) = \frac{M_\theta([\mathbf{t}:\mathbf{t}']\mid T)}{M_\theta(\mathbf{t}\mid T)} &=  
\frac{\prod_{t'_i\in[\mathbf{t}:\mathbf{t}']}^{|[\mathbf{t}:\mathbf{t}']|} m_\theta(t'_i \mid [\mathbf{t}:\mathbf{t}']_{<i}(\mathcal{L}(T)))}{\prod_{t_i\in\mathbf{t}}^{|\mathbf{t}|} m_\theta(t_i \mid \mathbf{t}_{<i}(\mathcal{L}(T)))} \\
&= \prod_{t'_i\in\mathbf{t}'}^{|\mathbf{t}'|} m_\theta(t'_i \mid \mathbf{t'}_{<i}(\mathbf{t}(\mathcal{L}(T)))) \\
&= \prod_{t'_i\in\mathbf{t}'}^{|\mathbf{t}'|} m_\theta(t'_i \mid \mathbf{t'}_{<i}(\mathcal{L}(T'))) \\&= M_\theta(\mathbf{t}' \mid T')
\end{aligned}
\]
Hence, $M_\theta$ is $\mathscr{R}$-equivariant in proof distribution.
In general, however, when combined with explicit search, this symmetry is lost due to the finite sampling budgets -- a state reached once during the search may never be explored, leading to a non-transferable distribution of proofs.
}

\textbf{$\mathscr{R}_\iota$-invariance of success probability.}
\draft{The inductive biases of multi-step provers become even clearer when considering success invariance (\Cref{sec:semantic_consistency}).
Let $\mathscr{R}_\iota$ denote the rewriting category generated solely by the empty tactic sequence $\iota \in \mathcal{T}^*$.
Every arrow in $\mathscr{R}_\iota$ is then of the form $T \xrightarrow{\iota} T'$ and satisfies $\mathcal{L}(T') = \iota(\mathcal{L}(T)) = \mathcal{L}(T)$.
Hence, whenever $T \leftrightarrows T'$ in $\mathscr{R}_\iota$, both statements compile to the same initial proof state.
As a result, $M_\theta$ performs search from an identical starting point on both inputs, implying equal success probabilities.
Therefore, $M_\theta$ is $\mathscr{R}_\iota$-invariant in success probability. Notably, this argument does not require $M_\theta$ to be sequential.}

\begin{remark}
Next-tactic prediction introduces a more complex training objective, making optimization challenging and leading to weaker empirical performance.
A promising direction for future work is to combine these paradigms by integrating state-conditioned reasoning into single-pass approaches.
\end{remark}

\subsection{Capturing compositional behavior in Lean proofs}
\label[appendix]{app:additional_remarks}

In this work, we consider rewriting categories $\mathscr{R}$ in which each tactic sequence induces a distinct arrow.
Concretely, for arrows $\mathbf{t}, \mathbf{t}' : T \to T'$ in $\mathscr{R}$, if the underlying tactic sequences $\mathbf{t}$ and $\mathbf{t}'$ are syntactically distinct, then the corresponding arrows are treated as distinct.
Nevertheless, the formalization of the Lean statement--tactic space introduced in \Cref{sec:formalization_lean_space} can capture richer structural properties of this space, requiring provers to reason about the underlying structure of tactics rather than their syntactic form alone.

\begin{wrapfigure}{r}{0.36\textwidth}
\vspace{-1em}
    \centering

\definecolor{compBoxHdr}{HTML}{FFD7A6}
\definecolor{compBoxHdrText}{HTML}{8A4B08}
\definecolor{compBoxDraw}{HTML}{D88A2D}

\resizebox{0.85\linewidth}{!}{%
\begin{tikzpicture}[
    font=\sffamily\normalsize,
    tacticbox/.style={
        rectangle split, rectangle split parts=2,
        rectangle split part fill={compBoxHdr, white},
        rectangle split draw splits=true,
        rounded corners=3pt, draw=compBoxDraw, line width=0.4pt,
        fill=white,
        text width=5.7cm, inner xsep=6pt, inner ysep=4pt,
        align=left,
        font=\ttfamily\footnotesize,
        text=black,
    },
]
    \node[tacticbox] (t1) at (0,0) {%
        {\sffamily\bfseries\footnotesize\color{compBoxHdrText}Tactic $t_1$}
        \nodepart{two}%
        rw [Nat.sub\_add\_eq]};

    \node[tacticbox, below=0.25cm of t1] (t2) {%
        {\sffamily\bfseries\footnotesize\color{compBoxHdrText}Tactic $t_2$}
        \nodepart{two}%
        rw [Nat.sub\_add\_cancel]};

    \node[tacticbox, below=0.25cm of t2] (t12) {%
        {\sffamily\bfseries\footnotesize\color{compBoxHdrText}Tactic $[t_1 : t_2]$}
        \nodepart{two}%
        rw [Nat.sub\_add\_eq]\\
        rw [Nat.sub\_add\_cancel]};

    \node[tacticbox, below=0.25cm of t12] (t3) {%
        {\sffamily\bfseries\footnotesize\color{compBoxHdrText}Tactic $t_3$}
        \nodepart{two}%
        rw [Nat.sub\_add\_eq, Nat.sub\_add\_cancel]};
\end{tikzpicture}%
}
    \caption{Composition $[t_1 : t_2]$ has the same effect as $t_3$, despite them being syntactically different.}
    \label{fig:tactic_composition}

    \vspace{4em}
    \definecolor{compBoxHdr}{HTML}{F8C8C8}
\definecolor{compBoxHdrText}{HTML}{8B1E1E}
\definecolor{compBoxDraw}{HTML}{C44E4E}

\resizebox{0.85\linewidth}{!}{%
\begin{tikzpicture}[
    font=\sffamily\normalsize,
    tacticbox/.style={
        rectangle split, rectangle split parts=2,
        rectangle split part fill={compBoxHdr, white},
        rectangle split draw splits=true,
        rounded corners=3pt, draw=compBoxDraw, line width=0.4pt,
        fill=white,
        text width=5.7cm, inner xsep=6pt, inner ysep=4pt,
        align=left,
        font=\ttfamily\footnotesize,
        text=black,
    },
]
    \node[tacticbox] (t1) at (0,0) {%
        {\sffamily\bfseries\footnotesize\color{compBoxHdrText}Tactic $l_1$}
        \nodepart{two}%
        have h$_0$ : x * 3 $\equiv$ 1 [ZMOD 7] := by \\
        \hspace*{1.2em}norm\_num
        };

    \node[tacticbox, below=0.25cm of t1] (t2) {%
        {\sffamily\bfseries\footnotesize\color{compBoxHdrText}Tactic $l_2$}
        \nodepart{two}%
        have h$_0$ : x * 3 $\equiv$ 1 [ZMOD 7] := by \\
        \hspace*{1.2em}simpa [hx]
        };
\end{tikzpicture}%
}
    \caption{Tactics $l_1$ and $l_2$ may act differently on some statements.}
    \label{fig:lemma_equivalence}
    \vspace{-1em}
\end{wrapfigure}%

\draft{Consider the tactics $t_1$, $t_2$, and $t_3$ shown in \Cref{fig:tactic_composition}.
Tactics $t_1$ and $t_2$ rewrite the current proof state using the Mathlib lemmas \leaninline{Nat.sub_add_eq} and \leaninline{Nat.sub_add_cancel}, respectively, while $t_3$ applies both these rewriting rules, within a single tactic.}

\draft{Hence, the sequential composition $[t_1 : t_2]$ has the same effect as applying $t_3$ directly.
Although both induce the same transformation function $d : \Omega \to \Omega$ on Lean states, the concatenation $[t_1 : t_2]$ and $t_3$ remain syntactically distinct, as shown in \Cref{fig:tactic_composition}.}
A categorical formulation allows these transformations to be identified in $\mathscr{R}$ by requiring that, for any objects $T,T' \in \mathscr{R}$, the morphisms $[t_1 : t_2] : T \to T'$ and $t_3 : T \to T'$ coincide, without imposing any equivalence between $t_3$ and either $t_1$ or $t_2$ individually. This notion of compositional equivalence cannot be naturally expressed using graph- or group-based representations.

The example in \Cref{fig:tactic_composition} concerns tactic sequences that are globally equivalent as functions $\Omega \to \Omega$. Category-theoretic structure also enables a finer notion of local equivalence. Consider the tactics $l_1$ and $l_2$ from \Cref{fig:lemma_equivalence}. Both introduce the auxiliary statement that $3x \equiv 1 \pmod 7$, but construct the proof differently. Consequently, they are not extensionally equivalent: for example, $l_2$ requires the presence of a hypothesis \leaninline{hx} in the proof state. Nevertheless, categorical structure allows these transformations to be identified whenever both act successfully on a given state, capturing an informal notion of ``having proved an auxiliary result inside the main proof, no matter how, the model should continue in the same way''.

These examples further motivate the use of category theory as a formal framework for the space of Lean statements and tactics.
In contrast to graph- or group-based representations, categorical structure naturally captures compositional and contextual equivalences between tactics, including both globally equivalent transformations and locally equivalent behaviors conditioned on proof states.
This additional expressive power enables modeling structural relationships between tactics that extend beyond purely syntactic or functional representations.
\subsection{Rewriting categories of formal languages}\label[appendix]{app:formal_languages}

In this subsection, we leave a brief remark on how rewriting categories could be used to model symmetries of formal languages, beyond the mathematical domain that we focus on in this work.

We consider formal languages equipped with a syntax and a semantics, where the syntax specifies which expressions are valid, and the semantics assigns meaning to expressions. Programming and logical languages are the canonical examples \citep{winskel1993formal, enderton2001mathematical, girard1989proofs}. A central concern therein is checking whether two given expressions are semantically equivalent. In programming languages, this could be if two programs map every input to the same output, or if they execute in the same way. In logic, this could be if each formula is derivable from the other. Unfortunately, these are largely undecidable \citep{rice1953classes, turing1936computable}.

A workaround is using a tractable notion of equivalence based on a set of rewriting rules that takes an expression and syntactically transforms it without altering its semantics of interest. If two expressions can be made identical through rewriting, we call them syntactically equivalent, which implies they are semantically equivalent (as in \Cref{sec:semantic_consistency}, the converse does not necessarily hold). This approach is usually implemented with \emph{rewriting systems} that reduce expressions to normal forms that cannot be further rewritten, and compare those to check equivalence \citep{book1993string, baader1998term}.
This idea is used in compiler optimization \citep{tate2009equality}, program refactoring \citep{lubin2024equivalence}, and (classical) automated theorem proving \citep{rossel2024equality}.

Rewriting categories (\Cref{sec:rewriting_categories}) can be used to describe rewriting systems. An (abstract) rewriting system consists of a set $A$, elements called objects or signatures, together with a binary relation on $A$, denoted by $\to$, called the rewrite relation \citep{book1993string}. Then we can construct a rewriting category $\mathscr{R}$ by
\begin{itemize}[noitemsep,nolistsep]
\item taking as its objects the set $A$,
\item taking as its arrows the reflexive transitive closure $\overset{*}{\to}$ of $\to$, i.e. the transitive closure of $\to \cup\, \{{\rm id}\}$, where ${\rm id}$ is the identity relation.
\end{itemize}
In other words, we can lift a rewriting system to a rewriting category by composing its rewriting relations and identity. Thus, the relationship between them is analogous to that between a group and its generating set. An $\mathscr{R}$-invariant map, defined in naturality sense (as in \Cref{app:success_invariance}), preserves syntactic equivalence by construction, and hence preserves semantic equivalence.

From a rewriting system, one can take a subset of its relations to obtain a restriction. For example, if the relations of a system $(A,\to)$ is an indexed union $\to\,=\,\to_1\cup\to_2$ of two rewriting rules, one can take $(A,\to_1)$ as a restriction. It would yield sound rewrites, but more incomplete as some semantically equivalent expressions would no longer be syntactically equivalent. For the respective categories $\mathscr{R}$ and $\mathscr{R}_1$, this translates to a subcategory relation $\mathscr{R}_1\subseteq\mathscr{R}$, and so a hierarchy of invariances as in \Cref{sec:semantic_consistency}. We exemplify with lambda calculus, a foundational example in programming and logic.

\begin{example}\label{example:lambda_calculus}
The untyped lambda calculus consists of lambda expressions defined by the syntax $e::=x\mid (\lambda x.e) \mid (e\,e)$ where $x$ ranges over a namespace. Its rewriting system can be identified as $(\Lambda, \to)$, where $\Lambda$ is the set of all lambda expressions, and the relations $\to=\to_\alpha\!\cup\!\to_\beta\!\cup\!\to_\eta$ consist of
\begin{itemize}[noitemsep, nosep]
\item $\alpha$-conversion (variable renaming) that rewrites $\lambda x.M$ to $\lambda y.M[x:=y]$ if $y$ is not free in $M$,
\item $\beta$-reduction (function application) that reduces an expression of the form $(\lambda x.t)s$ to $t[x:=s]$.
\item $\eta$-conversion (extensionality) that converts between $\lambda x.fx$ and $f$ if $x$ does not appear free in $f$.
\end{itemize}
Hence, there is a restricted rewriting system $(\Lambda,\to_\alpha)$ that only performs variable renaming. Between the rewriting categories $\mathscr{R}=(\Lambda,\overset{*}{\to})$ and $\mathscr{R}_\alpha=(\Lambda,\overset{*}{\to_\alpha})$, this translates to $\mathscr{R}_\alpha\subseteq\mathscr{R}$.
\end{example}

In the context of deep learning, \Cref{example:lambda_calculus} illustrates how approaches that aim to analyze or improve the robustness of neural networks on formal languages primarily focusing on variable renaming \citep{mirzadeh2024gsm,pei2023exploiting} yield weaker invariances than more general rewritings, such as in our work.
\section{Extended methodology}
\label[appendix]{app:extended_methodology}

This section provides additional details on the experimental methodology.
We first describe the construction of the miniF2F-rw dataset, which augments miniF2F \cite{zheng2022miniff} with equivalent problem reformulations (Appendix \ref{sec:minif2f-generation}).
We then detail the sampling mechanism $\mu(\cdot \mid T, K)$ used in the test-time ensemble setting (\Cref{sec:sampling-mechanism}) and report ablations over its design choices (\Cref{sec:sampling-mechanism}) and the number of ensemble variants in miniF2F-rw evaluations (\Cref{app:ablation_number_of_ensembles}).
The codebase, containing implementations of rewriting ensembles and the miniF2F-rw dataset can be found at: \href{https://github.com/kolejnyy/rw-ensembles}{https://github.com/kolejnyy/rw-ensembles}.

\subsection{Practical challenges of statement reformulation in Lean}
\label{sec:challenges_of_statement_reformulation}

Before detailing our methodology, we emphasize that statement rewriting in Lean is substantially more challenging than in natural language settings.
In NLP, query rewriting is a standard technique for improving the robustness and reasoning capabilities of LLMs \citep{ma2023query,zhou2023leasttomost,obrien2024improving}.
Such approaches rely on the assumption that semantically equivalent queries admit interchangeable solutions: answers to rewritten queries can typically be reused directly or mapped back to the original formulation with minimal effort.
For example, the question ``\textit{What is the capital of Germany?}'' may be rewritten as ``\textit{Which city in Germany is the capital?}'', while preserving both meaning and answerability.
Moreover, modern LLMs exhibit strong understanding of natural language, making the generation of such rewrites fairly simple.

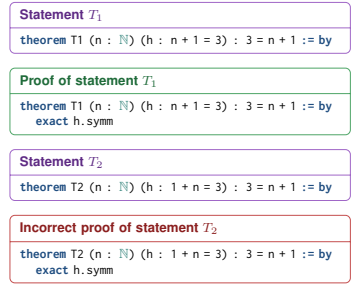
\begin{wrapfigure}{r}{0.34\textwidth}
    \centering
%

\definecolor{nitThmHdr}{HTML}{E7CCF6}
\definecolor{nitThmHdrText}{HTML}{5E2A84}
\definecolor{nitThmDraw}{HTML}{A067C7}
\definecolor{compBoxHdr}{HTML}{F8C8C8}
\definecolor{compBoxHdrText}{HTML}{8B1E1E}
\definecolor{compBoxDraw}{HTML}{C44E4E}
\definecolor{corrHdr}{HTML}{D3F0D7}
\definecolor{corrHdrText}{HTML}{1F6B3A}
\definecolor{corrDraw}{HTML}{4EA567}

\resizebox{0.95\linewidth}{!}{%
\begin{tikzpicture}[
    font=\sffamily\normalsize,
    >={Stealth[length=2.6mm,width=2.0mm]},
    thmbox/.style={
        rectangle split, rectangle split parts=2,
        rectangle split part fill={nitThmHdr, white},
        rectangle split draw splits=true,
        rounded corners=3pt, draw=nitThmDraw, line width=0.4pt,
        fill=white,
        text width=6.8cm, inner xsep=6pt, inner ysep=4pt,
        align=left,
        font=\ttfamily\footnotesize,
        text=black,
    },
    prfbox/.style={
        rectangle split, rectangle split parts=2,
        rectangle split part fill={compBoxHdr, white},
        rectangle split draw splits=true,
        rounded corners=3pt, draw=compBoxDraw, line width=0.4pt,
        fill=white,
        text width=6.8cm, inner xsep=6pt, inner ysep=4pt,
        align=left,
        font=\ttfamily\footnotesize,
        text=black,
    },
    corrbox/.style={
        rectangle split, rectangle split parts=2,
        rectangle split part fill={corrHdr, white},
        rectangle split draw splits=true,
        rounded corners=3pt, draw=corrDraw, line width=0.4pt,
        fill=white,
        text width=6.8cm, inner xsep=6pt, inner ysep=4pt,
        align=left,
        font=\ttfamily\footnotesize,
        text=black,
    },
    tacticpill/.style={
        draw=orange!60!black!50, rounded corners=4pt, fill=yellow!35, line width=0.4pt,
        inner sep=5pt, font=\ttfamily\footnotesize,
        drop shadow={shadow xshift=0pt, shadow yshift=-0.7pt, opacity=0.18, fill=black!60},
    },
]
    \node[thmbox] (orig) at (0,0) {%
        {\sffamily\bfseries\footnotesize\color{nitThmHdrText}Statement $T_1$}
        \nodepart{two}%
        \lkw{theorem} T1 (n : \lty{$\mathbb{N}$}) (h : n + 1 = 3) : 3 = n + 1 \lkw{:= by}
        };
    \node[corrbox, below=0.3cm of orig] (prf) {%
        {\sffamily\bfseries\footnotesize\color{corrHdrText}Proof of statement $T_1$}
        \nodepart{two}%
        \lkw{theorem} T1 (n : \lty{$\mathbb{N}$}) (h : n + 1 = 3) : 3 = n + 1 \lkw{:= by} \\
        \hspace*{1.2em}\lkw{exact} h.symm 
        };

    \node[thmbox, below=0.3cm of prf] (aug) {%
        {\sffamily\bfseries\footnotesize\color{nitThmHdrText}Statement $T_2$}
        \nodepart{two}%
        \lkw{theorem} T2 (n : \lty{$\mathbb{N}$}) (h : 1 + n = 3) : 3 = n + 1 \lkw{:= by}
        };

    \node[prfbox, below=0.3cm of aug] (prf) {%
        {\sffamily\bfseries\footnotesize\color{compBoxHdrText}Incorrect proof of statement $T_2$}
        \nodepart{two}%
        \lkw{theorem} T2 (n : \lty{$\mathbb{N}$}) (h : 1 + n = 3) : 3 = n + 1 \lkw{:= by} \\
        \hspace*{1.2em}\lkw{exact} h.symm 
        };
\end{tikzpicture}%
}
    \caption{Lean proofs do not transfer between equivalent statements.}
    \label{fig:proof_transfer}
    \vspace{-1em}
\end{wrapfigure}%
In Lean theorem proving, these assumptions no longer apply. Proof correctness is determined by a syntactically strict formal system, where even minor modifications to a statement may invalidate an existing proof.
As a result, semantically equivalent statements do not necessarily admit interchangeable proofs, and transferring proofs between related formulations can require non-trivial transformations. Consider the statements $T_1$ and $T_2$ in \Cref{fig:proof_transfer}.
Although \leaninline{n + 1} and \leaninline{1 + n} are mathematically equivalent, they are syntactically distinct terms, and Lean does not identify them automatically.
Consequently, the proof term \leaninline{h.symm}, which proves $T_1$, is rejected as a proof of~$T_2$.
This illustrates that meaningfully transforming a problem while remaining capable of recovering a proof of the original is non-trivial.

A sound augmentation mechanism can be implemented in two distinct ways:
\begin{itemize}[nosep]
\item \textbf{Out-of-Lean.} Analogously to augmentation methods in natural language tasks, an LLM may generate rewrites of the original statement directly.
However, this approach additionally requires constructing a Lean certificate that transfers proofs of the augmented statement back to the original problem.
In general, such certificates are non-trivial to obtain and may themselves require LLM inference, which needs to be accounted for under fixed computation budgets.
\item \textbf{Within-Lean.} Alternatively, augmentations can be generated directly within Lean by exploring the space of proof states via tactics. This avoids using LLM calls outside the core proving process, but introduces two challenges: selecting meaningful rewrite rules without excessively complicating the proof state, and defining a mapping from proof states back to theorem statements, since LLM provers operate on statements, whereas Lean applies tactics on states.
\end{itemize}

In our test-time sampling framework (Appendix~\ref{sec:sampling-mechanism}), we adopt the latter approach.
We sample transformations from the large collection of unconditional premises available in Mathlib, comprising over 30{,}000 lemmas.
This yields non-trivial statement rewritings, as illustrated in \Cref{fig:examples_of_solved_rewrites}, while retaining a principled mechanism for proof transfer through tactic composition.

\subsection{Generation of the miniF2F-rw dataset}
\label{sec:minif2f-generation}

To systematically demonstrate semantic inconsistency among LLM-based Lean provers, we construct \textit{miniF2F-rw}, a dataset of 5726 theorems ($2922$ validation, $2804$ test) obtained as semantics-preserving reformulations of problems from the {miniF2F} benchmark.
Each variant is paired with Lean certificates of \texttt{variant} $\Leftrightarrow$ \texttt{seed}, ensuring that any proof of the augmented statement can be translated into a proof of the original, and vice versa.
We further apply manual checks to remove variants that introduce spurious complexity or trivialize the task.
An overview of the pipeline is shown in \Cref{fig:dataset_generation}.

\textbf{Initial variant generation.}
First, we generate candidate augmentations by prompting GPT-5.4 to produce $5$-$15$ rewrites of each miniF2F statement.
For each seed theorem, we then sample $20\%$ of its variants and apply a second transformation, again using GPT-5.4, to rename variables and hypothesis names.
The exact prompts used in this process are listed in Appendix \ref{app:prompts}.
At the end of this stage, we obtain over 7500 statements for subsequent verification.

\textbf{Certifying equivalence.}
To ensure semantic equivalence between augmented statements and their seeds, we employ a two-stage validation procedure.
First, we manually check each variant and discard those that alter the content in an unreasonably complicated (e.g., replacing \leaninline{1} with \leaninline{0.factorial}) or simplistic (e.g., simplifying algebraic expressions) way.

Second, we generate Lean certificates establishing \texttt{variant} $\Leftrightarrow$ \texttt{seed}.
Each such implication is decomposed into a series of simpler subproblems.
Writing a seed theorem $T$ and its variant $T'$ as:
\[
T := (C_1 \:\land\: \dots \:\land\:C_k) \implies G \qquad\text{and}\qquad  T' := (C_1' \:\land\: \dots \:\land\:C_l') \implies G'
\]
for the direction \texttt{variant} $\Rightarrow$ \texttt{seed}, it suffices to show:
\[
\begin{aligned}
    (C_1 \:\land\: \dots \:\land\:C_k) \implies& C'_i \qquad& \text{for all } i \in {1, \dots, l} \\
    G' \implies& G 
\end{aligned}
\]
Given these components, a proof of $T$ can be obtained by reconstructing the premises of $T'$, invoking $T'$ (simulated via \leaninline{sorry}), and finally reducing $G'$ to $G$ (see \Cref{fig:certificate example} for a code example).
The inverse implication can be derived analogously, by flipping $T$ and $T'$.
This methodology resembles that of BEq \citep{liu2025beq}, which calls a prover to produce ``a proof of $T$ using $T'$''.
However, BEq only checks whether the name of $T'$ was mentioned in the proof, without assessing whether it was meaningfully used, while our decomposition enforces semantic correctness.

\begin{figure}
    \centering
    \includegraphics[width=0.95\linewidth]{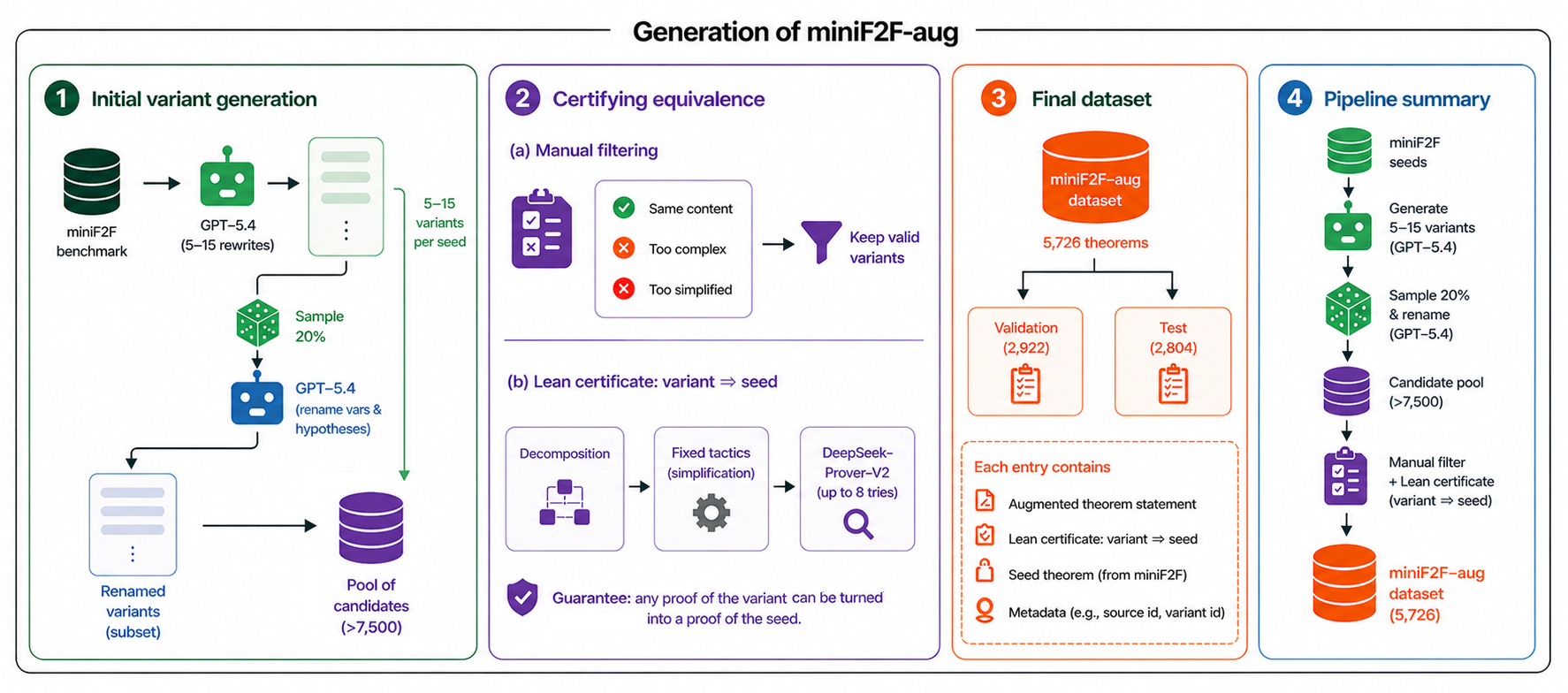}
    \caption{Process of generation of the miniF2F-rw dataset.}
    \label{fig:dataset_generation}
\end{figure}

To close each of these subgoals, we initially apply a set of fixed simplification tactics (see the proof of \leaninline{goal_to_goal} in \Cref{fig:certificate example}).
If this fails, we invoke DeepSeek-Prover-V2 \citep{ren2025deepseek} with up to $8$ attempts per subproblem.
Since the decomposed goals are significantly simpler than the original statements, this budget suffices in most cases.
Finally, we manually check the failure instances and repair them when a small number of augmentations of a specific problem are successfully verified.

\begin{figure}
    \centering
\begin{lean}
theorem mathd_algebra_405_v2 (S : Finset ℕ)
  (h : ∀ x, x \in S ↔ 0 < x \and (x + 2) ^ 2 < 20) :
  2 = S.card := by
  sorry

theorem cond_1 (S : Finset ℕ)
  (h : ∀ x, x \in S ↔ 0 < x \and x ^ 2 + 4 * x + 4 < 20)
  : ∀ x, x \in S ↔ 0 < x \and (x + 2) ^ 2 < 20 := by
  intro x
  simp_all only [Finset.mem_filter, Finset.mem_range, Finset.mem_univ,
    true_and_iff, and_true_iff, Nat.cast_ofNat, Nat.cast_one,
    Nat.cast_add, Nat.cast_mul, Nat.cast_sub, Nat.cast_pow]
  constructor <;> intro h <;>
  (try constructor) <;>
  (try simp_all [sq]) <;>
  (try ring_nf at * <;> omega) <;>
  (try nlinarith) <;>
  (try linarith)

theorem goal_to_goal (S : Finset ℕ) (aug_goal : 2 = S.card) : S.card = 2
  := by
  try decide
  try native_decide
  try trivial
  try linarith
  try nlinarith
  try omega
  try norm_num
  try field_simp
  try simp
  try ring_nf
  try ring
  try exact?

theorem mathd_algebra_405 (S : Finset ℕ)
  (h : ∀ x, x \in S ↔ 0 < x \and x ^ 2 + 4 * x + 4 < 20) :
  S.card = 2 := by
  have aug_0 : ∀ x, x \in S ↔ 0 < x \and (x + 2) ^ 2 < 20 := cond_1 S h
  have aug_goal : 2 = S.card := mathd_algebra_405_v2 S aug_0
  exact goal_to_goal S aug_goal
\end{lean}
    \caption{An example of a Lean certificate, proving \leaninline{mathd_algebra_405_v2} $\Rightarrow$ \leaninline{mathd_algebra_405}.}
    \label{fig:certificate example}
\end{figure}

\subsection{Description of the test-time sampling mechanism}
\label{sec:sampling-mechanism}

Beyond the controlled miniF2F-rw setting, where augmentations are precomputed and verified, we consider a test-time ensemble framework, generating reformulations online.
The procedure consists of three stages: rewriting, composition, and reranking, described below and illustrated in \Cref{fig:test-time-sampling-process}.

\begin{figure}
    \centering
    \includegraphics[width=0.99\linewidth]{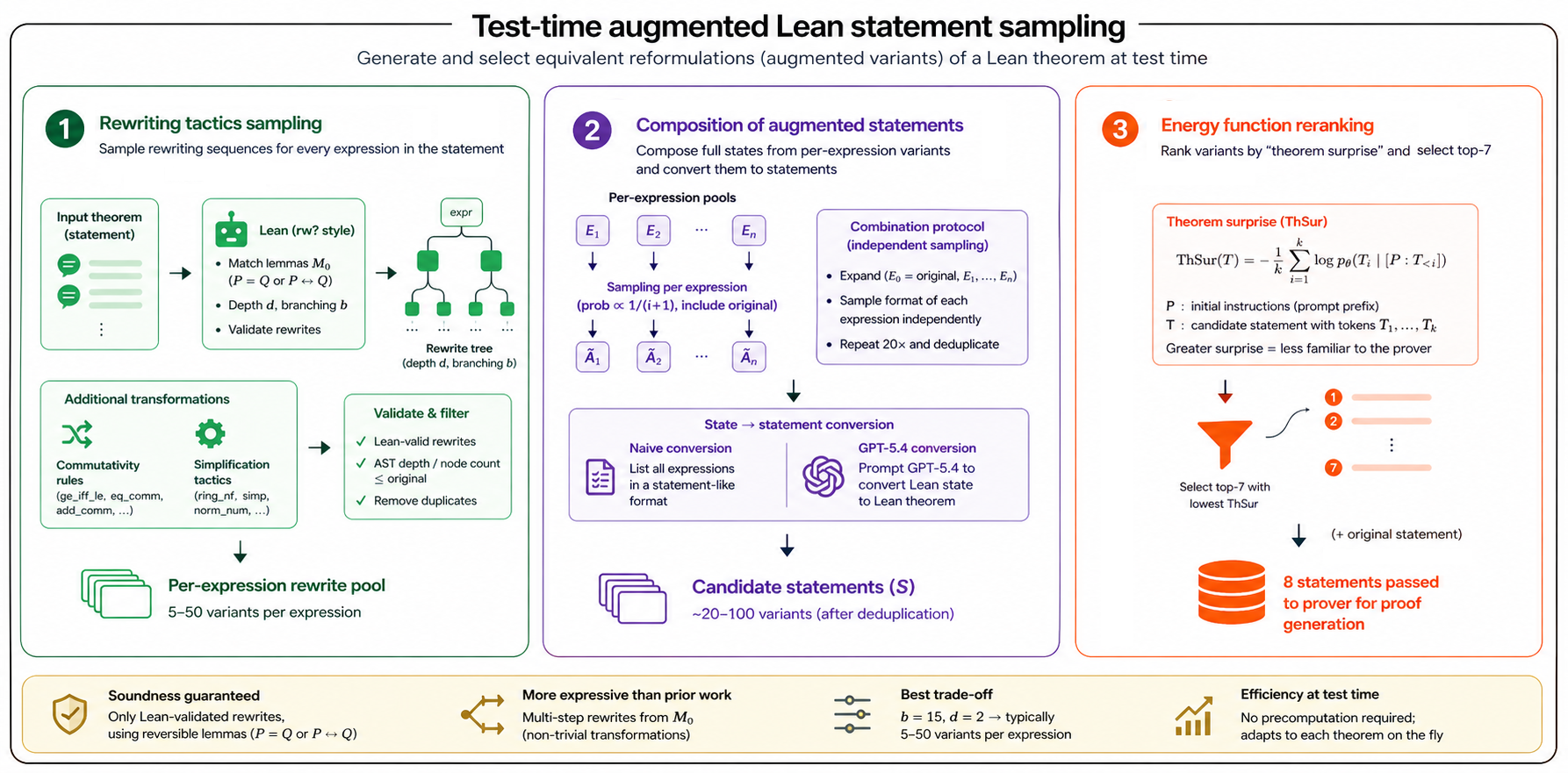}
    \vspace{-0.5em}
    \caption{Sampling process for the test-time ensemble mechanism.}
    \label{fig:test-time-sampling-process}
    \vspace{-0.5em}
\end{figure}

\textbf{Rewriting tactics sampling.}
In the first stage, we use Lean to sample rewrite sequences applicable to each expression in the input statement.
For a given expression, we construct sequences of $(L_1, \dots, L_d)$ lemmas drawn from Mathlib \citep{mathlib2020}.
To guarantee semantic equivalence, we restrict attention to the subset $\mathbb{M}_0$ of unconditioned lemmas of the form $P=Q$ or $P\Leftrightarrow Q$, whose application is inherently reversible.
Despite this constraint, the resulting transformation space strictly generalizes prior augmentation schemes \citep{tian2025evolprover,paliwal2020graph,li2021neurotactic,blaauwbroek2024graph2tac}, which rely on abstract syntax tree (AST) symmetries to ensure soundness.
In contrast, $\mathbb{M}_0$ induces a richer rewriting category $\mathscr{R}_{\mathbb{M}_0}$ that supports non-trivial transformations (see \Cref{fig:examples_of_solved_rewrites}), while preserving correctness through Lean tactic semantics.
This provides a tractable, more complete approximation to the full tactic rewriting category $\mathscr{R}_\mathcal{T}$.

To sample the sequences of rewriting rules, we use a methodology similar to the \leaninline{rw?} tactic, yet with increased exploration breadth and depth.
At each step, candidate lemmas from $\mathbb{M}_0$ are matched against the expression’s AST, ranked using the \leaninline{rw?}'s utility heuristic, and validated via application.
We pick the top-$b$ valid rewrites and recursively expand them up to depth $d$. This yields up to $b^d$ candidate sequences and corresponding rewritten expressions.

To control complexity, we apply AST-based filtering, discarding rewrites whose tree depth or node count exceeds that of the original expression.
We then deduplicate syntactically equal outputs arising from different rewrite paths.
After validation, we found that $b=15$ and $d=2$ provide a favorable performance-to-efficiency trade-off, usually producing $5$–$50$ distinct variants per expression.

As the search space over $\mathbb{M}_0$ remains large, limiting to top-$b$ candidates may omit semantically meaningful rewrites.
To mitigate this, we explicitly augment the pool with selected commutativity rules (\leaninline{ge_iff_le, ← ge_iff_le, ← gt_iff_lt, gt_iff_lt, eq_comm, add_comm}, \leaninline{mul_comm}, \leaninline{and_comm}, \leaninline{or_comm}) and standard normalization tactics (\leaninline{ring_nf}, \leaninline{simp}, \leaninline{norm_num}, \leaninline{field_simp}).
The~resulting outputs are incorporated into the final set of candidate rewrites.

\textbf{Composition of augmented statements.}
Given the set of rewritten versions for each expression (premises and goal), we construct augmented theorem statements suitable for downstream provers.
This involves two design choices: (i) how to select rewrites for each statement, and (ii) how to map rewritten expressions back to a coherent theorem.

For rewrite selection, a simple baseline modifies a single expression per variant, keeping all others unchanged.
While straightforward, this approach becomes restrictive for statements with many components, as all but one remain unchanged in each augmentation.
To increase diversity, we instead sample a rewrite for each expression independently from its candidate set.
Concretely, for an expression $E$ with ranked rewrites $(E_1, \dots, E_n)$, we form $(E_0 = E, E_1, \dots, E_n)$ and sample $E_i$ with probability proportional to $\frac{1}{i+1}$.
Applying this procedure across all expressions yields a sampled state.
Repeating it $20$ times and removing duplicates produces a diverse set of augmented states.

We then convert each state into a theorem statement.
This step is non-trivial, as Lean’s internal representation may differ from its surface syntax (e.g., implicit type annotations or notation changes such as \leaninline{Real.pi} vs.\ $\pi$).
While syntactic correctness is not strictly required (generated proofs are ultimately checked against the original statement after tactics-based transfer), poorly formed outputs can degrade prover performance.
We therefore consider two strategies.
A naive approach directly linearizes the state into a statement-like format.
A more advanced state-to-statement function is realised by prompting an LLM to convert provided Lean states to theorems.
Empirically, the latter yields higher-quality outputs and resolves common inconsistencies, as confirmed by ablations.
In our experiments, we use GPT-5.4 \cite{openai2025gpt54} for this conversion (prompt available in Appendix \ref{sec:state_to_statement_prompt}).
This choice is primarily for convenience: the mapping is highly systematic, and a small LLM trained on $(\text{state}, \text{statement})$ pairs should suffice. 
We leave this simplification for future work.

\textbf{Energy function reranking.}
Given a set $S$ of retrieved statements, we select a subset of $7$ variants to pass to the prover.
Inspired by orbit canonicalization methods \citep{kaba2023equivariance, singhal2025test, lubin2024equivalence}, we rank candidates using an energy function that approximates the success probability $s_\theta$ of prover $p_\theta$.
As estimating $s_\theta$ directly would require full proof attempts, we introduce a lightweight proxy, called \emph{theorem surprise}.

The underlying intuition is that LLM-based provers perform better on in-distribution inputs.
We therefore prioritize statements that appear more “natural” to the model, as measured by their token-level likelihood under $p_\theta$.
Concretely, for a statement $T$, to generate a proof for $T$, $p_\theta$ is prompted with input $[P:T]$, where $P$ is the initial set of instructions ("\textit{You are an expert in Lean4...}").
Letting $T_1, \dots, T_l$ be the tokens representing $T$, we define the theorem surprise of $T$ as:
\[
\text{ThSur}(T) = - \frac{1}{l} \sum_{i = 1}^l \log p_\theta(T_i \mid [P:T_{<i}])
\]
where $p_\theta(T_i \mid [P:T_{<i}])$ is the probability of $p_\theta$ generating token $T_i$ after seeing $P$ and the first $i-1$ tokens of $T$.
Higher values indicate less familiar inputs.
We rank candidates by $\text{ThSur}$ and select the seven lowest-scoring statements, which, together with the original problem, are passed to $p_\theta$ for proof generation.

\begin{table}[h]
    \caption{Success rate per ensemble style on miniF2F-valid (with DeepSeek-Prover-V2)}
    \label{tab:ensemble-abl-results}
    \centering
    \small
    \setlength{\tabcolsep}{10pt}
    \begin{tabular}{lcccc}
    \toprule
    \textbf{Budget} & \textbf{8} & \textbf{16} & \textbf{32} & \textbf{64} \\
    \midrule
    DSP-V2 (baseline) & $78.4$ & $78.8$ & $79.1$ & $79.5$  \\
    \midrule
    \midrule
    DSP-V2 (ThSur + simp) & $78.7$ & $79.2$ & $79.7$ & ${80.1}$  \\
    DSP-V2 (ThSur + simp + excomm + comb20) & $78.9$ & $79.6$ & ${80.1}$ & ${80.6}$ \\
    DSP-V2 (ThSur + simp + excomm + comb20 + LLM) & $\bm{79.7}$ & $\bm{80.5}$ & $\bm{81.1}$ & $\bm{81.5}$ \\
    \bottomrule
    \end{tabular}
    \vspace{-0.4em}
\end{table}

\textbf{Ablation on the sampling protocol.}
We iteratively refine the sampling procedure via ablations on miniF2F-valid (\Cref{tab:ensemble-abl-results}).
All variants use theorem surprise for reranking and include rewrites obtained from standard simplification tactics.
Augmenting the rewrite pool with explicit commutativity rules and switching to the combination-based state construction (sampling 20 variants) yields consistent improvements.
Finally, replacing the naive state-to-statement conversion with GPT-5.4 further boosts performance, highlighting the importance of well-formed input representations for prover success.

\subsection{Ablation on the number of sampled rewrites}
\label[appendix]{app:ablation_number_of_ensembles}

\Cref{prop:monotonicity} predicts that ensemble performance should increase with the number of sampled augmentations.
We empirically validate this by evaluating ensemble sizes of $2$, $4$, and $8$ on miniF2F-rw.
\Cref{fig:minif2f-ensemble-ablation} reports validation and test success rates as a function of the sample budget for each non-reasoning prover in our main setup.
Across all models, PASS@$k$ improves monotonically with the ensemble size, consistent with the theoretical prediction.
Notably, for Goedel-Prover-DPO and DeepSeek-Prover-V2, using 4 variants already surpasses the seed baseline despite the performance gap between seed and random augmentations, with further gains observed at 8 samples.
Based on these results, we use 8 variants in all miniF2F-rw experiments.
\begin{figure}[t]
    \centering
    \vspace{-1.9em}
    \includegraphics[width=0.95\linewidth]{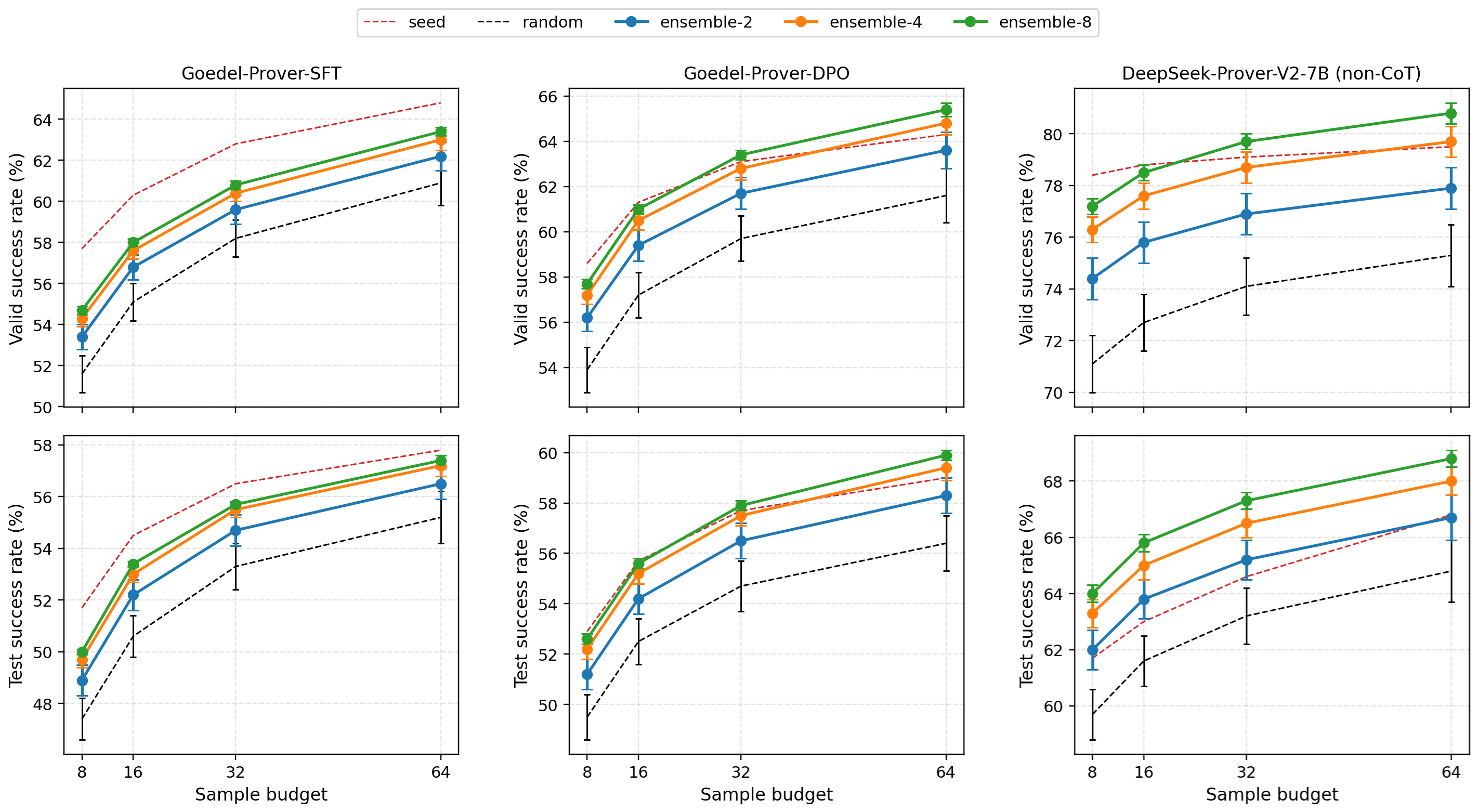}
    \caption{PASS@$k$ success rates on miniF2F-rw for ensemble sizes $2$, $4$, and $8$ (solid lines), with single seed and random-sample baselines as narrower dashed lines. Ensemble and random curves include $\pm$ one estimated standard error.}
    \label{fig:minif2f-ensemble-ablation}
\end{figure}

\subsection{Adaptation of PASS@\textit{k} to ensembles}
In its standard formulation, given $C$ successes over $N$ attempts, PASS@$k$ is measured as:
\[
\text{PASS@}k = 1 - \frac{\binom{N-C}{k}}{\binom{N}{k}}
\]
representing the probability of seeing at least one success across randomly chosen $k$.
For an ensemble, running on $K$ variants with success-total attempt pairs $(N_1, C_1), \dots, (N_K, C_K)$, we report:
\[
\text{PASS}_{ens}@k = 1 - \prod_{i=1}^K \frac{\binom{N_i-C_i}{k/K}}{\binom{N_i}{k/K}}
\]

\subsection{Testing environment specification and hyperparameters.}
\label[appendix]{sec:hyperparameters}

\textbf{Hardware specifications.}
We perform proof generation on a machines equipped with Intel Platinum 8628 CPUs and NVIDIA H100 GPUs (80GB VRAM).
Each model is allocated 8 CPU cores, 64GB RAM, and a single H100 GPU. Under this setup, non-reasoning models generate one proof attempt every 0.5-1.3s on average, while models with explicit informal reasoning require 7-10s per attempt.

Proof verification is conducted on machines with Intel Platinum 8628 or Intel Xeon Gold 6326 CPUs. This stage is highly parallelizable; we dispatch separate jobs per problem, each using 8 CPU cores and 64GB RAM. Verification latency is approximately 1s per attempt.
Overall, for the miniF2F-rw benchmark (5,700 problems, 64 attempts each), a single non-reasoning prover requires approximately 4–5 GPU-days for generation and an additional \textasciitilde4 CPU-days for Lean-based verification.

\textbf{Model hyperparameters.}
All LLM provers used in our experiments are prompted with the original prompts from their repositories. 
We use vLLM \citep{kwon2023vllm} for more efficient generation, and sampling temperature of $0.7$ for all models.
The maximum number of output tokens is set to $2000$ for non-reasoning, and $8192$ for reasoning provers.
For calls to GPT-5.4 \cite{openai2025gpt54}, we fix the temperature at~$0.2$.

\textbf{API costs.}
We use GPT-5.4 to augment problem statements during construction of the miniF2F-rw dataset. The cost per statement ranges from $\$0.01$ to $\$0.05$, totaling approximately $\$14.50$ for $488$ benchmark problems.
At test time, GPT-5.4 is additionally employed to convert states into statements for ensembling. This incurs a total cost of $\$18.19$ across all datasets (miniF2F, ProofNet, and Ineq-Comp), corresponding to approximately $\$0.009$ per problem. As discussed in Appendix~\ref{sec:sampling-mechanism}, this transformation can be implemented using a local model, mitigating cost in practical deployments.
\section{Additional experiments with reasoning models}
\label[appendix]{app:reasoning_experiments}

\begin{table*}
    \centering
    \caption{PASS@$k$ results of selected reasoning LLM provers on the miniF2F-rw benchmark.}
    \label{tab:minif2f-rw-reasoning-results}
    \tiny
    \vspace{0.5em}
    \setlength{\tabcolsep}{2pt}
\begin{tabular}{l@{\hspace{12pt}}l@{\hspace{12pt}}llllllll}
\toprule
 & & \multicolumn{4}{c}{\textit{minif2f-aug-valid}} & \multicolumn{4}{c}{\textit{minif2f-aug-test}} \\
\midrule
\textbf{Model} & \textbf{Variant} & $k=1$ & $k=8$ & $k=16$ & $k=32$ & $k=1$ & $k=8$ & $k=16$ & $k=32$ \\
\midrule
\multirow{4}{*}{Goedel-Prover-V2-8B} & {seed} & $\bm{68.8}$ & $\bm{81.1}$ & $\bm{82.5}$ & $\bm{83.2}$ & $\bm{58.9}$ & $\underline{74.6}$ & $76.4$ & $77.5$ \\
& {random} & $\underline{65.2{\scriptstyle \pm0.7}}$ & $79.5{\scriptstyle \pm0.6}$ & $81.2{\scriptstyle \pm0.6}$ & $82.2{\scriptstyle \pm0.6}$ & $\underline{56.0{\scriptstyle \pm0.7}}$ & $72.9{\scriptstyle \pm0.7}$ & $75.3{\scriptstyle \pm0.8}$ & $76.7{\scriptstyle \pm0.8}$ \\
& {controlled} & -- & $\underline{80.6{\scriptstyle \pm0.1}}$ & $\underline{81.9{\scriptstyle \pm0.1}}$ & $\underline{82.8{\scriptstyle \pm0.1}}$ & -- & $\bm{75.0{\scriptstyle \pm0.1}}$ & $\bm{77.0{\scriptstyle \pm0.1}}$ & $\underline{78.3{\scriptstyle \pm0.1}}$ \\
& {test-time} &  -- & $80.4$ & $\underline{81.9}$ & $\underline{82.8}$ & -- & $74.4$ & $\underline{76.7}$ & $\bm{78.6}$ \\
\midrule
\multirow{4}{*}{DeepSeek-Prover-V2-7B} &{seed} & $\bm{67.4}$ & $\bm{78.1}$ & $\bm{79.7}$ & $80.7$ & $\bm{58.0}$ & $\bm{71.0}$ & $\bm{73.2}$ & $\bm{75.0}$ \\
&{random} & $\underline{62.7{\scriptstyle \pm0.8}}$ & $75.9{\scriptstyle \pm0.7}$ & $77.7{\scriptstyle \pm0.8}$ & $79.3{\scriptstyle \pm1.0}$ & $\underline{53.4{\scriptstyle \pm0.8}}$ & $68.5{\scriptstyle \pm0.7}$ & $70.8{\scriptstyle \pm0.7}$ & $73.0{\scriptstyle \pm0.9}$ \\
&{controlled} & -- & $\underline{77.7{\scriptstyle \pm0.1}}$ & $\underline{79.5{\scriptstyle \pm0.2}}$ & $\bm{81.0{\scriptstyle \pm0.2}}$ & -- & $70.3{\scriptstyle \pm0.1}$ & $72.3{\scriptstyle \pm0.1}$ & $\underline{74.2{\scriptstyle \pm0.1}}$ \\
& {test-time} & -- & $77.5$ & $79.3$ & $\bm{81.0}$ & -- & $\underline{70.9}$ & $\underline{72.9}$ & $\underline{74.2}$ \\
\midrule
\multirow{4}{*}{Kimina-Prover-Distill-8B} & {seed} & $\bm{61.9}$ & $\underline{73.4}$ & $\underline{75.4}$ & $\underline{76.6}$ & $\bm{55.0}$ & $68.2$ & $70.0$ & $71.3$ \\
& {random} & $\underline{57.6{\scriptstyle \pm0.9}}$ & $70.9{\scriptstyle \pm0.9}$ & $73.6{\scriptstyle \pm0.9}$ & $75.8{\scriptstyle \pm1.1}$ & $\underline{52.2{\scriptstyle \pm0.8}}$ & $66.5{\scriptstyle \pm0.8}$ & $68.9{\scriptstyle \pm0.8}$ & $70.7{\scriptstyle \pm0.9}$ \\
& controlled & -- & $\bm{74.2{\scriptstyle \pm0.2}}$ & $\bm{76.5{\scriptstyle \pm0.2}}$ & $\bm{78.3{\scriptstyle \pm0.2}}$ & -- & $\underline{68.8{\scriptstyle \pm0.1}}$ & $\underline{71.0{\scriptstyle \pm0.1}}$ & $\underline{72.5{\scriptstyle \pm0.1}}$ \\
& test-time &  -- & $72.0$ & $74.5$ & $76.3$ & -- & $\bm{70.1}$ & $\bm{72.6}$ & $\bm{74.4}$ \\
\bottomrule
\end{tabular}
\end{table*}

\subsection{Robustness of reasoning provers on miniF2F-rw}

We evaluate three state-of-the-art open-source reasoning LLM provers: Goedel-Prover-V2-8B \cite{lin2025goedel_v2}, Kimina-Prover-Distill-8B \cite{kimina_prover_2025} and DeepSeek-Prover-V2-7B \cite{ren2025deepseek} (in CoT mode).
We run each prover for $n=32$ attempts on all variants from miniF2F-rw and report PASS@$k$ for $k\in\{1,8,16,32\}$. 

The results are presented in \Cref{tab:minif2f-rw-reasoning-results}.
Although the degradation is less severe than for non-reasoning provers (\Cref{tab:minif2f-rw-results}), suggesting that an explicit reasoning phase improves robustness to formal rewrites, all models still consistently prefer the seed statements over their augmented variants.
For example, on the miniF2F-rw-test split, Goedel-Prover-V2-8B drops from $58.9$ to $56.0$ PASS@1 and from $77.5$ to $76.7$ PASS@8.
The effect is more pronounced for DeepSeek-Prover-V2-7B, whose PASS@1 declines from $67.4$ to $62.7$ on the validation split, and from $58.0$ to $53.4$ on the test set.
This indicates that current reasoning-oriented provers remain sensitive to syntactic formulation and do not exhibit full invariance under semantics-preserving transformations.
Consequently, these results further motivate the structural bias criteria proposed in this work, even for reasoning models.

Consistent with the observations in \Cref{sec:results}, ensembling across semantically equivalent variants reduces performance variance and consistently improves over selecting a single random sample.
This trend holds across all evaluated models, datasets, and sampling budgets. Importantly, both the \textbf{controlled} and \textbf{test-time} sampling protocols remain at least competitive with the \textbf{seed} baselines in every configuration.
For Kimina-Prover-Distill-8B, they substantially outperform the default model, increasing PASS@32 on the test split from $71.3\%$ to $74.4\%$.
Overall, these results suggest that, even for reasoning-based formal provers, allocating the inference budget across multiple equivalent representations of the input theorem can yield significant performance gains.

\subsection{Comparisons at matched token counts on miniF2F}
\label{app:minif2f-tokens}

\begin{table}[t]
    \caption{Average token usage per inference mode per attempt for DeepSeek-Prover-V2-7B.}
    \label{tab:minif2f-tokens}
    \centering
    \footnotesize
    \setlength{\tabcolsep}{10pt}
    \begin{tabular}{lrr}
    \toprule
    \textbf{Model} & \textbf{miniF2F-valid} & \textbf{miniF2F-test} \\
    \midrule
    {DeepSeek-Prover-V2-7B} & 346 & 395 \\
    {DeepSeek-Prover-V2-7B (CoT)} & 3138 & 3763 \\
    \bottomrule
    \end{tabular}
    \vspace{-1em}
\end{table}

\Cref{tab:minif2f-tokens} contains average token counts per problem on miniF2F.
Since reasoning mode uses 9-10$\times$ more tokens on average per problem, in \Cref{tab:minif2f-rw-results,tab:minif2f-rw-reasoning-results}, a comparison at roughly matched token counts within DeepSeek-Prover-V2 results would be comparing reasoning with budget $k$ against ensemble with budget $8k$. In these comparisons, ensemble is competitive with and often surpasses reasoning, showing its utility as an inference-time scaling method in this setting. While this matches token counts, we remark that in terms of the computational cost of forward passes, reasoning would be more inefficient, because in modern LLMs, generating a single long chain is worse than a set of shorter generations due to lack of parallelizability and the quadratic cost of self-attention.
\section{Proofs}
\label[appendix]{app:proofs}

\subsection{Functoriality of the proof functor}
\label[appendix]{sec:proof-functoriality}

Let $\mathscr{R}$ be a rewriting category. Recall the definition of the (stochastic) proof functor map on $\mathscr{R}$: 

\textbf{Definition \ref{def:proof_functor}.} 
Let $\mathbb{D}_+(\mathcal{T}^*)$ be the set of all distributions over the set $\mathcal{T}^*$ of Lean proofs, which assign non-zero probability to each element $\mathbf{t}\in\mathcal{T}^*$. 
The \emph{(stochastic) proof functor} on $\mathscr{R}$ is the functor $\mathcal{P}:\mathscr{R}\to\mathbf{Set}$ that maps each object $T\in \Stmt$ to $\mathcal{P}(T) = \mathbb{D}_+(\mathcal{T}^*)$,
and whose action on arrows $T \xrightarrow{\;\mathbf{t}\;} T'$ of $\mathscr{R}$ is defined as:
\[
\mathcal{P}(\mathbf{t})(\nu)(\mathbf{t}') \propto \nu([\mathbf{t}\!:\! \mathbf{t}'])\qquad \forall \nu \in \mathbb{D}_+(\mathcal{T}^*), \mathbf{t'}\in \mathcal{T}^*
\]

We claim that the so-defined map is indeed a functor.

\begin{proposition}
    The map from \Cref{def:proof_functor} is a well-defined functor $\mathcal{P}:\mathscr{R}\to\mathbf{Set}$.
\end{proposition}

\textit{Proof. }
    First of all, note that $\mathbb{D}_+(\mathcal{T}^*) \in \Set$, so the function of $\mathcal{P}$ on objects is well-defined.
    
    Pick any $\mathbf{t}, \mathbf{t}' \in \mathcal{T}^*$ and $\nu \in \DpT$. The concatenation $\bf[t: t']$ is also an element of $\mathcal{T}^*$. Moreover, by positivity of $\nu$, the sum of $\nu([\mathbf{t} : \mathbf{t}''])$ over $\mathbf{t}'' \in \mathcal{T}^*$ is greater than 0. Hence, we can define a distribution $\nu' \in \DpT$ with probabilities proportional to $\nu([\mathbf{t}: \cdot])$, as follows:
    \[
    \nu'(\mathbf{t'}) = \frac{\nu([\mathbf{t}:\mathbf{t'}])}{\sum_{\mathbf{t}''\in\mathcal{T}^*} \nu([\mathbf{t} : \mathbf{t}''])} \qquad \forall \: \mathbf{t}' \in \mathcal{T}^*
    \]
    being precisely the distribution $\mathcal{P}(\mathbf{t})(\nu)(\mathbf{t}')$ from \Cref{def:proof_functor}. Therefore, the action of $\mathcal{P}$ on arrows of $\mathscr{R}$ is also well-defined. 
    It now suffices to show that $\mathcal{P}$ satisfies the functor conditions, i.e. preserves:
    \begin{itemize}
        \item \textit{identities}: Let $T\in \Stmt$ be a statement.
        The identity arrow $1_T$ in $\mathscr{R}$ is associated with the~application of the empty tactic sequence $\iota$. Unfolding the definition from above, we obtain:
        \[
        \mathcal{P}(1_T)(\nu)(\mathbf{t}') = \frac{\nu([\iota:\mathbf{t'}])}{\sum_{\mathbf{t}''\in\mathcal{T}^*} \nu([\iota : \mathbf{t}''])} = \nu([\iota : \mathbf{t}']) = \nu(\mathbf{t}')\qquad \forall \: \nu \in \DpT, \mathbf{t}'\in \mathcal{T}^*
        \]
        meaning that $\mathcal{P}(1_T)(\nu) = \nu$ for all $\nu\in \DpT$. Consequently, $\mathcal{P}(1_T) = 1_{\DpT} = 1_{\mathcal{P}(T)}$.

        \item \textit{composition}: Let $T \xrightarrow{\;\mathbf{t}\;} T' \xrightarrow{\;\mathbf{t}'\;} T''$ be arrows in $\mathscr{R}$. We need to show $\mathcal{P}(\mathbf{t}') \circ \mathcal{P}(\mathbf{t}) = \mathcal{P}(\mathbf{t}' \circ \mathbf{t})$. 
        Let $\nu \in \DpT$ be a distribution and let $\mathbf{t}'' \in \mathcal{T}^*$ be a proof. Then:
        \[
        \left(\mathcal{P}(\mathbf{t}') \circ \mathcal{P}(\mathbf{t}) \right) (\nu)(\mathbf{t}'') = \left(\mathcal{P}(\mathbf{t}') \right) (\mathcal{P}(\mathbf{t})(\nu))(\mathbf{t}'') \propto (\mathcal{P}(\mathbf{t})(\nu))([\mathbf{t}' : \mathbf{t}'']) \propto \nu([\mathbf{t} : \mathbf{t}' : \mathbf{t}''])
        \]
        and similarly
        \[
        \mathcal{P}(\mathbf{t}' \circ \mathbf{t})(\nu)(\mathbf{t}'') \propto \nu([(\mathbf{t}' \circ \mathbf{t}) : \mathbf{t}'']) = \nu([\mathbf{t} : \mathbf{t}' : \mathbf{t}''])
        \]
        Since both $\left(\mathcal{P}(\mathbf{t}') \circ \mathcal{P}(\mathbf{t}) \right)\! (\nu)$ and $\mathcal{P}(\mathbf{t}' \circ \mathbf{t})(\nu)$ are distributions over $\mathcal{T}^*$, by the result above, they must be equal. As $\nu$ was arbitrary, we can conclude that  $\mathcal{P}(\mathbf{t}') \circ \mathcal{P}(\mathbf{t}) = \mathcal{P}(\mathbf{t}' \circ \mathbf{t})$. \qed
    \end{itemize}

\clearpage

\subsection{Proof of \Cref{prop:invariance_limit}}
\label[appendix]{sec:proof-invariance-limit}

To prove $\mathscr{R}$-invariance of the ensemble framework in the limit of samples, we assume the following:

\textbf{\Cref{ass:sampling_coverage}.} Let $T \sim_\mathscr{R} T'$ and let $T'' \in \mathscr{R}$ be a statement. Suppose that for some $K$, the~probability $\mathbb{P}(T'' \!\in \mu(\cdot \mid T, K)) $ of sampling $T''$ is greater than $0$. Then: 
\[
\lim_{K\to\infty} \mathbb{P}(T'' \!\in \mu(\cdot \mid T, K)) = \lim_{K\to\infty} \mathbb{P}(T'' \!\in \mu(\cdot \mid T', K)) = 1
\]

As mentioned in \Cref{sec:sample-limit-consistency}, this assumption is true whenever the sampling protocol explores the space of reachable states thoroughly enough, or when it canonicalizes the input theorem, transforming it into a specific representative of its equivalence class.
In our methodology described in \Cref{sec:methodology}, we propose a mixture of these ideas: sampling various statements reachable from the input theorem, followed by re-ranking using an energy function.
This way, we approximate canonicalization by optimizing an objective function over a sampled subspace of reachable reformulations.

Under \Cref{ass:sampling_coverage}, we can prove that the ensemble framework is $\mathscr{R}$-invariant in success probability in the limit of computational budget and the number of samples:

\textbf{\Cref{prop:invariance_limit}}
Let $n:\mathbb{N} \rightarrow \mathbb{N}$ be a function such that $\lfloor n(x)/x\rfloor\to\infty$ as $x \to\infty$. Then,
\[
\lim_{K\to\infty}\bar{s}_{\theta,\mu}^{(K,n(K))}(T) \;=\; \lim_{K\to\infty}\bar{s}_{\theta,\mu}^{(K,n(K))}(T')\qquad\text{for every }T\sim_{\mathscr{R}}T'.
\]
so limiting ensemble prover is $\mathscr{R}$-invariant in success probability.

\begin{proof}
Recall that the probability $\bar{s}_{\theta,\mu}^{(K,N)}(T)$ of success of an $(N,K)$-ensemble on statement $T$ is:
\begin{equation}
\label{eq:ens-probability}
\bar{s}_{\theta,\mu}^{(K,N)}(T) \;=\; 1 - \mathbb{E}_{(T_1,\ldots,T_K)\:\sim\:\mu(\cdot\mid T, K)}\!\left[\,\prod_{i=1}^K \bigl(1-s_\theta(T_i)\bigr)^{\lfloor N/K\rfloor}\,\right].
\end{equation}

Let $T\sim_\mathscr{R} T'$ be $\mathscr{R}$-equivalent statements.
Denote by $\mathbb{P}_\mu(T'' \mid T, K)$ the probability that $T''$ will be among the samples drafted by $\mu(\cdot \mid T, K)$.
Suppose there exists some $T''$ with $s_\theta(T'') > 0$, such that for some $K$, $\mathbb{P}_\mu(T'' \mid T, K) > 0$.
Then, by the \Cref{ass:sampling_coverage}, we have:
\[
\lim_{K\to\infty}\mathbb{P}_\mu(T'' \mid T, K) = \lim_{K\to\infty} \mathbb{P}_\mu(T'' \mid T', K) = 1
\]
Since each term $\bigl(1-s_\theta(T_i)\bigr)^{\lfloor N/K\rfloor}$ in \Cref{eq:ens-probability} is bounded above by $1$, we can write:
\[
\begin{aligned}
\mathbb{E}_{(T_1,\ldots,T_K)\:\sim\:\mu(\cdot\mid T, K)}\!\left[\,\prod_{i=1}^K \bigl(1-s_\theta(T_i)\bigr)^{\lfloor N/K\rfloor}\,\right] \quad\leq\quad &\mathbb{P}_\mu(T'' \mid T, K) \cdot \bigl(1-s_\theta(T'')\bigr)^{\lfloor N/K\rfloor} \\&+ \big(1 - \mathbb{P}_\mu(T'' \mid T, K)\big)
\end{aligned}
\]
Indeed, if $T''$ appears among the sampled states, which happens with probability $\mathbb{P}_\mu(T'' \mid T, K)$, the~product is at most $\bigl(1-s_\theta(T'')\bigr)^{\lfloor N/K\rfloor}$.
Otherwise, with the probability of $1 - \mathbb{P}_\mu(T'' \mid T, K)$, we can crudely bound it by $1$.

With this observation, we can derive a lower bound for the limit from the problem statement:
\[
\begin{aligned}
    \lim_{K\to\infty}\bar{s}_{\theta,\mu}^{(K,n(K))}(T) &= \lim_{K\to\infty} \left(1 - \mathbb{E}_{(T_1,\ldots,T_K)\:\sim\:\mu(\cdot\mid T, K)}\!\left[\,\prod_{i=1}^K \bigl(1-s_\theta(T_i)\bigr)^{\lfloor n(K)/K\rfloor}\,\right] \right) \\
    &= 1 - \lim_{K\to\infty} \left(\mathbb{E}_{(T_1,\ldots,T_K)\:\sim\:\mu(\cdot\mid T, K)}\!\left[\,\prod_{i=1}^K \bigl(1-s_\theta(T_i)\bigr)^{\lfloor n(K)/K\rfloor}\,\right] \right) \\
    &\geq 1 - \lim_{K\to\infty}\left(\mathbb{P}_\mu(T'' \mid T, K) \cdot \bigl(1-s_\theta(T'')\bigr)^{\lfloor n(K)/K\rfloor} + \big(1 - \mathbb{P}_\mu(T'' \mid T, K)\big)\right)
\end{aligned}
\]
Now, we already know that
\[
\lim_{K\to\infty} \mathbb{P}_\mu(T'' \mid T, K) = 1
\]
Furthermore, as $n(K)/K \to \infty$ as $K\to\infty$ and $0 \leq 1 - s_\theta(T'') < 1$, we can deduce that
\[
\lim_{K\to\infty} \bigl(1-s_\theta(T'')\bigr)^{\lfloor n(K)/K\rfloor} = 0
\]
Combining these observations, we can derive that
\[
\begin{aligned}
    \lim_{K\to\infty}\bar{s}_{\theta,\mu}^{(K,n(K))}(T)
    &\geq 1 - \lim_{K\to\infty}\left(\mathbb{P}_\mu(T'' \mid T, K) \cdot \bigl(1-s_\theta(T'')\bigr)^{\lfloor n(K)/K\rfloor} + \big(1 - \mathbb{P}_\mu(T'' \mid T, K)\big)\right) \\
    &= 1 - (1 \cdot 0 + (1 - 1)) \\
    &= 1
\end{aligned}
\]
As by definition, $\lim_{K\to\infty}\bar{s}_{\theta,\mu}^{(K,n(K))}(T) \leq 1$, we must have $\lim_{K\to\infty}\bar{s}_{\theta,\mu}^{(K,n(K))}(T) = 1$.
Since the proof of this property did not involve any properties specific to $T$ apart from the fact that $\lim_{K\to\infty}\mathbb{P}_\mu(T'' \mid T, K) = 1$, using $\lim_{K\to\infty} \mathbb{P}_\mu(T'' \mid T', K) = 1$, we can derive an analogous result for $T'$, giving us:
\[
\lim_{K\to\infty}\bar{s}_{\theta,\mu}^{(K,n(K))}(T) \;=\; \lim_{K\to\infty}\bar{s}_{\theta,\mu}^{(K,n(K))}(T') = 1
\]

For the second case, suppose that for all $T''$ with $\mathbb{P}_\mu(T'' \mid T, K) > 0$ for some $K$, the probability $s_\theta(T'')$ that the prover $p_\theta$ solves $T''$ equals $0$. This clearly leads to:
\[
\bar{s}_{\theta,\mu}^{(K,n(K))}(T) \;=\; 0 \qquad \text{for all } K
\]
and consequently
\[
\lim_{K\to\infty}\bar{s}_{\theta,\mu}^{(K,n(K))}(T) \;=\; 0
\]
Then, an analogous result has to hold for $\lim_{K\to\infty}\bar{s}_{\theta,\mu}^{(K,n(K))}(T')$. Indeed, otherwise, there would exist some statement $S$ with $s_\theta(S)>0$ such that $\mathbb{P}_\mu(S \mid T', K) > 0$ for some $K$. By \Cref{ass:sampling_coverage}, since $T'\sim_\mathscr{R} T$, this would imply that 
\[
\lim_{K\to\infty}\mathbb{P}_\mu(S \mid T, K) = \lim_{K\to\infty} \mathbb{P}_\mu(S \mid T', K) = 1
\]
so in particular, there would exist some value of $K$ such that $\mathbb{P}_\mu(S \mid T, K) > 0$.
This contradicts the assumption that $\mu(\cdot \mid T, K)$ cannot sample any statement with positive provability chance $s_\theta$. Hence,
\[
\lim_{K\to\infty}\bar{s}_{\theta,\mu}^{(K,n(K))}(T) \;=\; \lim_{K\to\infty}\bar{s}_{\theta,\mu}^{(K,n(K))}(T') = 0
\]
Either way,
\[\lim_{K\to\infty}\bar{s}_{\theta,\mu}^{(K,n(K))}(T) \;=\; \lim_{K\to\infty}\bar{s}_{\theta,\mu}^{(K,n(K))}(T')\]
\end{proof}

\clearpage
\subsection{Proof of \Cref{prop:monotonicity}}
\label[appendix]{sec:proof-monotonicity}
As the second theoretical result about the proposed ensembling framework, we establish that under minor symmetry assumptions on the success rate of equivalent theorems, increasing the number of considered reformulations improves the chances of finding a correct proof.

The assumption used for this proposition states that for a given prover $p_\theta$, statement $T$ and a rewriting category $\mathscr{R}$, the probability $s_\theta(T')$ of producing a correct proof for $T' \in [T]_\mathscr{R}$ can be viewed as a random variable sampled from a distribution $P_{[T]_\mathscr{R}}$, characteristic for the equivalence class $[T]_\mathscr{R}$.
Intuitively, this means that given two equivalent formulations of the same problem, we cannot predict which one of them will be more favorable to $p_\theta$.
This is motivated by the empirical examples of inconsistency of modern formal provers (\Cref{fig:gcd_example}, \Cref{sec:results}).

Before proving the main theorem, we first establish a helpful auxiliary result:

\begin{lemma}
    \label{lem:holder_x}
    Let $X$ be a $[0,1]$-valued random variable, and let $K \leq K'$ be divisors of $N$. Then,
    \[
    \mathbb{E}\Big[\, X^{\lfloor N/K'\rfloor}\,\Big]^{K'} \leq \mathbb{E}\Big[\, X^{\lfloor N/K\rfloor}\,\Big]^{K}
    \]
\end{lemma}
\begin{proof}
    If $K = K'$, we trivially get equality, so suppose $K < K'$.
    Let $a = { N/K'}$ and $b = {N/K}$.
    Then, $a \!<\! b$.
    Applying H\"older's inequality with exponents $r = \frac{b}{a} = \frac{K'}{K} > 1$ and $s = \frac{b}{b-a} > 1$, gives:
    \[
    \mathbb{E}[X^a] = \mathbb{E}[X^a \cdot 1^b] \leq (\mathbb{E}[(X^a)^r])^\frac{1}{r}\cdot(\mathbb{E}[(1^b)^s])^\frac{1}{s} = (\mathbb{E}[X^b])^\frac{a}{b}
    \]
    Therefore, raising both sides to the power $K'$, we obtain:
    \[
    \mathbb{E}\Big[\, X^{\lfloor N/K'\rfloor}\,\Big]^{K'} 
    = (\mathbb{E}[X^a]) ^ {K'}
    \leq \big((\mathbb{E}[X^b])^\frac{a}{b}\big)^{K'}
    = \mathbb{E}\Big[\, X^{\lfloor N/K\rfloor}\,\Big]^{\frac{a}{b}\cdot K'}
    = \mathbb{E}\Big[\, X^{\lfloor N/K\rfloor}\,\Big]^{K}
    \]
\end{proof}

\textbf{\Cref{prop:monotonicity}}
Let $\mathscr{R}$ be reciprocal, and $T\in\mathscr{R}$ be a statement.
Assume that for any $T' \sim _\mathscr{R} T$, the success probability $s_\theta(T') \in [0,1]$ can be viewed as an independent random variable sampled from some (prior) distribution $P_{[T]_\mathscr{R}}$. Then, for any budget $N$ and divisors $K\le K'$ of $N$,
\[
\bar{s}_{\theta,\mu}^{(K,N)}(T) \;\le\; \bar{s}_{\theta,\mu}^{(K',N)}(T),
\]
with strict inequality whenever $s_\theta$ is non-constant on the support of $\mu(\cdot\mid T, K) \subseteq [T]_{\mathscr{R}}$.

\begin{proof}
    Recall that if $\mathscr{R}$ is reciprocal, any statement $T'$ reachable from $T$ by arrows of $\mathscr{R}$ is, in fact, $\mathscr{R}$-equivalent to $T$.
    Hence, $\mu(\cdot \mid T, K)$ samples from $[T]_\mathscr{R}$.
    The success rate $s_\theta(T_i)$ of the $i^\text{th}$ sampled statement can be seen as a random variable $S_i$ following the distribution $S_i \sim \PTR$.
    Moreover, these random variables are independent, so we can rewrite the probability $\bar{s}_{\theta,\mu}^{(K,N)}(T)$ as:
    \[
    \begin{aligned}
    \bar{s}_{\theta,\mu}^{(K,N)}(T)
    &= 1 - \mathbb{E}_{\{T_1,\ldots,T_K\}\:\sim\:\mu(\cdot\mid T, K)}\!\left[\,\prod_{i=1}^K \bigl(1-s_\theta(T_i)\bigr)^{\lfloor N/K\rfloor}\,\right] \\
    &= 1 - \mathbb{E}_{\{T_1,\ldots,T_K\}\:\sim\:\mu(\cdot\mid T, K)}\!\left[\,\prod_{i=1}^K \bigl(1-S_i\bigr)^{\lfloor N/K\rfloor}\,\right] \\
    &= 1 - \mathbb{E}_{\{T_1,\ldots,T_K\}\:\sim\:\mu(\cdot\mid T, K)}\!\left[\mathbb{E}\left[\,\prod_{i=1}^K \bigl(1-S_i\bigr)^{\lfloor N/K\rfloor}\,\right]\right] \\
    &= 1 - \mathbb{E}_{\{T_1,\ldots,T_K\}\:\sim\:\mu(\cdot\mid T, K)}\!\left[\prod_{i=1}^K \mathbb{E}\Big[\, \bigl(1-S_i\bigr)^{\lfloor N/K\rfloor}\,\Big]\right] \\
    &= 1 - \mathbb{E}_{\{T_1,\ldots,T_K\}\:\sim\:\mu(\cdot\mid T, K)}\!\left[\prod_{i=1}^K \mathbb{E}\Big[\, \bigl(1-S\bigr)^{\lfloor N/K\rfloor}\,\Big]\right] \\
    &= 1 - \mathbb{E}\!\left[\prod_{i=1}^K \mathbb{E}\Big[\, \bigl(1-S\bigr)^{\lfloor N/K\rfloor}\,\Big]\right] \\
    &= 1 -  \mathbb{E}\Big[\, \bigl(1-S\bigr)^{\lfloor N/K\rfloor}\,\Big]^K \\
    \end{aligned}
    \]
    where $S\sim\PTR$. Analogously,
    \[
    \begin{aligned}
    \bar{s}_{\theta,\mu}^{(K',N)}(T)
    &= 1 - \mathbb{E}\Big[\, \bigl(1-S\bigr)^{\lfloor N/K'\rfloor}\,\Big]^{K'}
    \end{aligned}
    \]
    For notational convenience, let $X = 1-S$. It suffices to show:
    \[
    1 - \mathbb{E}\Big[\, X^{\lfloor N/K'\rfloor}\,\Big]^{K'} \geq 1 - \mathbb{E}\Big[\, X^{\lfloor N/K\rfloor}\,\Big]^{K}
    \]
    which follows trivially from \Cref{lem:holder_x}. Tracing back the argument using H\"older's inequality, we can additionally conclude that the equality holds if random variables $X^{\lfloor N/K\rfloor}$ and $1$ are proportional, which is equivalent to requiring $X$ to be constant.
    This, in turn, translates to $s_\theta$ being constant on the support $\mu(\cdot \mid T, K)$.
\end{proof}
\section{Prompts}
\label[appendix]{app:prompts}

\subsection{Rewrite generation prompt for GPT}
\label[appendix]{sec:rewrite_prompt}

\begin{lstlisting}
You are an expert Lean4 mathematician.
Task:
Generate {variants_min}-{variants_max} alternative theorem statements that are semantically equivalent 
to the provided Lean theorem. Produce as many high-quality distinct rewrites as possible, up to the maximum.
Hard constraints:
1) Keep the same variables and hypotheses names exactly as in the input.
2) Keep the same theorem name stem and append suffixes `_v2`, `_v3`, ... in increasing order.
3) Do not include proofs; every theorem must end exactly at `:= by` with nothing after 
(`sorry`, tactics, or extra lines).
4) Each rewrite must remain mathematically equivalent to the original theorem.
5) Avoid trivial whitespace-only edits. Prefer algebraic/logical rewrites.
6) DO NOT substitute hypothesis values/terms into other hypotheses or the goal (e.g. if `b = 30`, do not replace `b` with `30`).
7) Do not weaken/strengthen the theorem or add junk conjuncts/disjuncts (forbidden: `\and True`, `\iff False`, extra unrelated assumptions/goals).
Diversity requirement:
- Include several SIMPLE rewrites that only use symmetry/commutativity style transformations.
- Examples of SIMPLE rewrites:
  * `a = b` -> `b = a`
  * `a < b` -> `b > a`
  * `2 * x = y` -> `y = x * 2`
  * `u + v = c` -> `v + u = c`
  * `x ^ 3` -> `x * x * x`
  * `x ^ 2` -> `x * x`
  * `a * b` -> `b * a`
- Across the whole set, rewrite every hypothesis and the goal at least once when possible.
- Also include non-trivial but equivalent reformulations (not only simple ones).
Coverage checklist (must satisfy across produced variants):
- At least 3 variants must rewrite power notation (`^ 2` or `^ 3`) into repeated multiplication.
- At least 3 variants must use explicit symmetry/commutativity in products or equalities.
- At least 2 variants should rewrite division-style equalities into product-style equalities when safe.
  Example pattern: `a / b = c` -> `a = c * b` or `a = b * c`.
- At least 2 variants should alter parenthesization/ordering of multiplicative expressions.
- At least 2 variants should convert '|' divisibility statements into [MOD] or [ZMOD] expressions, and vice versa.
- Do not keep one key hypothesis text unchanged in almost all variants.
Allowed transformations include:
- equivalent rearrangements of equalities/inequalities
- equivalent predicate forms (`a | b` vs congruence/mod-0 forms when appropriate)
- syntactic rewrites using commutativity/associativity/distributivity
- rewriting `0 < x` as `x > 0`, rewriting `eg Even x` as `Odd x` when appropriate
- rewriting equations into equivalent multiplied/divided forms without changing assumptions
- replacing expressions with provably equivalent forms
Few-shot examples (from existing miniF2F 50x20 rewrites):
f{examples}
Now produce rewrites for this theorem:
f-- theorem id: {theorem_label}
```lean4
f{formal_statement.strip()}
```
Output format:
- Output ONLY Lean theorem declarations.
- Separate each theorem block with one blank line.
- Start from suffix `_v2`.
- Do not use markdown code fences.
- Never print `sorry` or a proof body after `:= by`.
\end{lstlisting}

\subsection{Variable renaming prompt for GPT}
\label[appendix]{sec:rename_prompt}

\begin{lstlisting}
You are an expert Lean 4 assistant.
Task:
Given a single Lean 4 theorem declaration (through `:= by`), produce a 
**renaming** of its variables and hypothesis names only. The mathematical 
meaning must stay the same (α-equivalence): only binder identifiers change.
Rules:
1) Keep the same theorem name as in the input.
2) Do not change the conclusion or hypothesis **types** except where a name 
appears as a binder (you may rename bound variables consistently).
3) `variable_map` maps old variable names → new names (e.g. `x` → `u`).
4) `hypothesis_map` maps old hypothesis names → new names (e.g. `h` → `h1` or `h0`).
5) If you rename nothing in a category, use an empty object `{}` for that map.
6) `original` must be exactly the input theorem string (after trimming).
7) `renamed` must be the full theorem line(s) ending with `:= by` only - no `sorry`, 
tactics, or proof body after `by`.
Output format:
- Output **only** one JSON object, no markdown fences, no commentary.
- Schema:
'  {original: ..., renamed: ..., variable_map: {...}, hypothesis_map: {...}}'
f-- theorem id: {label}
Input theorem:```lean4
f{stmt}
```
\end{lstlisting}
\clearpage
\subsection{State-to-statement conversion prompt}
\label[appendix]{sec:state_to_statement_prompt}
\begin{lstlisting}
You are an expert Lean4 formalization assistant.
Task:
Given an original Lean4 theorem statement for reference, and a tactic state created after 
applying some augmentations to the original theorem, create a theorem statement corresponding 
to the provided augmented state.
Requirements:
1) Keep theorem name and binder structure aligned with the original statement unless the augmented state requires changes.
2) Preserve hypothesis order from the augmented state.
3) Preserve the augmented state's mathematics exactly (do not simplify away changes).
4) Replace shorthand notation with explicit Lean forms when appropriate (e.g. `\sqrt2` -> `Real.sqrt 2`, `\pi` -> `Real.pi`).
5) Output must contain exactly one block in Lean markdown fence format:
   ```lean4
   theorem ... := by
   ```
6) Do not output any explanation or extra text outside that block.
7) Do not include a proof body, tactics, or sorry; stop exactly at `:= by`.
Original statement:
{original_statement}
Original state:
{original_state}
Augmented state:
{augmented_state}
\end{lstlisting}


\end{document}